\documentclass{article}

\PassOptionsToPackage{numbers, compress}{natbib}
\usepackage[preprint]{neurips_2026}

\usepackage[utf8]{inputenc} 
\usepackage[T1]{fontenc}    
\usepackage{hyperref}       
\usepackage{url}            
\usepackage{booktabs}       
\usepackage{amsfonts}       
\usepackage{nicefrac}       
\usepackage{microtype}      
\usepackage{xcolor}         
\usepackage{amsmath}
\usepackage{amssymb}
\usepackage{amsthm}
\usepackage{bm}
\usepackage{verbatim}
\usepackage{tikz}
\usetikzlibrary{arrows.meta, decorations.pathmorphing, calc, positioning}

\definecolor{hl-blue}{RGB}{0, 0, 180}
\definecolor{hl-green}{RGB}{0, 128, 0}
\definecolor{hl-red}{RGB}{196, 0, 0}
\hypersetup{
	colorlinks=true,       
	linkcolor=hl-red,      
	citecolor=hl-blue,     
	filecolor=magenta,     
	urlcolor=hl-green         
}

\newcommand{\ie}{\textit{i.e.}}
\newcommand{\eg}{\textit{e.g.}}
\newcommand{\ma}[1]{\ensuremath{\mathsf{#1}}}
\renewcommand{\vec}[1]{\ensuremath{\bm{#1}}}
\DeclareMathOperator*{\argmin}{arg\,min}

\newtheorem{definition}{Definition}[section]
\newtheorem{proposition}{Proposition}[section]

\newtheorem{lemma}{Lemma}[section]
\newtheorem{remark}{Remark} 

\title{A Generalized Tikhonov Layer for Interpretable-by-design Graph Neural Networks}

\author{%
	Nicolas Tremblay \\
	UiT the Arctic University of Norway\\
	\& CNRS, France \\
	\texttt{nicolas.tremblay@cnrs.fr}
	\And
	Benjamin Ricaud \\
	UiT the Arctic University of Norway\\
	\texttt{benjamin.ricaud@uit.no} 
	\And
	Filippo Maria Bianchi \\
	UiT the Arctic University of Norway\\
	\& NORCE Norwegian Research Centre \\
	\texttt{filippo.m.bianchi@uit.no}
}

\begin{document}

	\maketitle

	\begin{abstract}
		We propose the \emph{Tikhonov layer}, a graph neural network layer that is interpretable by design: once trained, its learned parameters directly reveal which node features and which aspects of the graph topology were leveraged for prediction. In practice, the layer's propagation matrix takes the closed-form $\ma{R} = (p(\ma{L})+\ma{Q})^{-1}\ma{Q}$, where $\ma{L}$ is the normalized graph Laplacian, $\ma{Q} = \mathrm{diag}(q_1,\ldots,q_n)$ a learnable diagonal matrix of positive node-importance scores, and $p(\cdot)$ a learnable polynomial. 
		For any input feature $\vec{x}$, the layer output $\ma{R}\vec{x}$ is the exact minimizer of a generalized graph Tikhonov problem that trades off node-level data fidelity against a topology-driven regularization penalty. 
		The learned pair $\{\{q_i\},p\}$ constitutes a built-in explanation: large $q_i$ indicates that node $i$'s own features drive the prediction, while small $q_i$ signals reliance on the local graph topology; the shape of $p$ reveals whether homophily, heterophily, or a band-pass response is exploited. 
		Expressivity is preserved by routing complexity through a dedicated, arbitrarily deep \emph{$\ma{Q}$-network} that produces the importance scores, while the Tikhonov layer itself remains transparent. 
		We prove that distinct node-importance matrices yield distinct propagation operators, structurally coupling the explanation to the computation. 
		Additionally, the Tikhonov layer provides, in a single layer, a global receptive field, mitigating both oversmoothing and oversquashing. 
		Experiments on standard graph classification benchmarks confirm that the model matches (and sometimes outperforms) opaque baselines while producing interpretable and faithful explanations.
	\end{abstract}
	
	\section{Introduction}
	Graph neural networks (GNNs) have seen rapid development over the last decade, advancing expressivity and performance on tasks such as graph and node classification. In graph machine learning, the data is twofold. On one hand is the structure, often represented as a graph $\mathcal{G}=(\mathcal{V},\mathcal{E})$ of $n=|\mathcal{V}|$ nodes interconnected by $m=|\mathcal{E}|$, possibly weighted, edges. On the other hand are the features, a matrix $\ma{X}\in\mathbb{R}^{n\times d}$ of which the $i$-th line represents the $d$-dimensional data defined over node $i$. In the standard layer formulation (with $\sigma$ a point-wise non-linearity, $\ma{W}$ a matrix of learnable weights):
	\begin{equation}
		\label{eq:GNN_layer}
		\ma{X}^{(k+1)} = \sigma\left(\ma{R}\ma{X}^{(k)}\ma{W}^{(k)}\right),
	\end{equation}
	the graph propagation matrix $\ma{R}\in\mathbb{R}^{n\times n}$ encodes how features from neighboring nodes are aggregated, and layers are stacked to progressively widen each node's receptive field.  Early choices for $\ma{R}$ were normalized versions of the adjacency or Laplacian matrices, or polynomials of these matrices; accounting for the class of so-called \textit{spectral} GNNs. More versatile models quickly appeared, under the generic name of \textit{message-passing} GNNs\footnote{The matrix form of equation~\eqref{eq:GNN_layer} is in fact too restrictive to encompass the full versatility of the message-passing paradigm, but it suffices for our purpose here and is much easier to read.}, where all entries $\ma{R}_{ij}$  could be independently learned. 
	For graph-level tasks, a permutation-invariant pooling operation (e.g., max, sum, or mean) followed by an MLP head maps node representations to predictions. We focus on graph-level classification and regression tasks for conciseness, though the setting also directly applies to node-level tasks.
	
	\noindent\textbf{Research questions.} 
	In the context of GNNs, interpretability is arguably even more complex than in classical neural networks: one must explain not only how the input features are progressively transformed and leveraged throughout the network, but also how the graph topology itself is exploited during the process. Are both the features and the topology used to solve the task? Or perhaps only the features or only the topology? Or perhaps only the features on some parts of the graph and the topology elsewhere? When the topology is used, how is it leveraged?

	\noindent\textbf{Proposed contribution.}
	We address these questions via the design of the propagation matrix $\ma{R}$ itself. 
	In its simplest, mono-channel form, our contribution constrains $\ma{R}$ to the form $\ma{R} = (p(\ma{L})+\ma{Q})^{-1}\ma{Q}$, with $\ma{L}$ the graph's normalized Laplacian\footnote{$\ma{L}=\ma{I}-\ma{D}^{-1/2}\ma{AD}^{-1/2}$ with $\ma{A}$ the adjacency matrix and $\ma{D}$ the diagonal matrix of degrees. 
		If $i$ is an isolated node (hence with non-invertible degree), we set $\forall j, \ma{L}_{ij}=\ma{L}_{ji}=0$. 
		Recall that the spectrum of $\ma{L}$ is contained in $[0,2]$.}, $\ma{Q} = \text{diag}(q_1, \ldots, q_n)$ a learnable diagonal matrix with positive entries, and $p(\cdot)$ a learnable  polynomial satisfying $0< p(\lambda)< 1$ on $[0,2]$. 
	This is not arbitrary:$\,\forall\vec{x}\hspace{-0.05cm}\in\hspace{-0.08cm}\mathbb{R}^n$, $\ma{R}\vec{x}$ is the unique closed-form solution of 
	the generalized graph Tikhonov problem:
	\begin{equation}
		\label{eq:VI}
		\ma{R}\, \vec{x} = \argmin_{\vec{z}\in\mathbb{R}^n} ~ \sum_{i\in\mathcal{V}}
		q_i(x_i-z_i)^2 + \vec{z}^\top p(\ma{L})\vec{z},
	\end{equation}
	in which each $q_i$ controls the relative importance of node $i$'s own features against the pressure exerted by the topology term $\vec{z}^\top p(\ma{L})\vec{z}$. 
	When $q_i$ is large, the original feature at node $i$ is preserved; when $q_i$ is small, the information at that node is mostly discarded in favor of a topology-driven solution. 
	The learned $p$ governs the nature of this influence~\cite{tremblay2018design}: an increasing $p$ penalizes high frequencies, pulling outputs toward smooth, homophilic values; a decreasing $p$ penalizes low frequencies, enabling heterophily; higher-degree polynomials yield band-pass responses.

	\begin{figure*}[t]
		\centering
		\hfill
		\begin{minipage}
			{0.3\textwidth}
			\small(\textbf{a) Node-level explanation}
			\centering
			\includegraphics[width=0.9\textwidth]{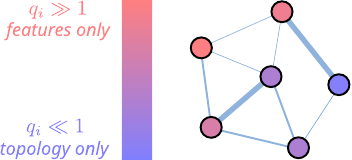}
		\end{minipage}\hfill
		\begin{minipage}
			{0.3\textwidth}
			\small(\textbf{b) Spectral explanation}
			\centering
			\includegraphics[width=0.9\textwidth]{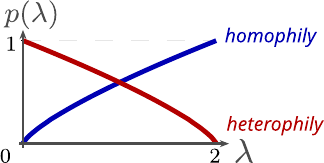}
		\end{minipage}
		\hfill~
		\caption{%
			\textbf{The explanation $\{\{q_i\}_{i\in\mathcal{V}}, p(\cdot)\}$.}
			\textbf{(a)}~After training, the learned~$q_i$ values provide a node-level explanation that can be plotted as a heatmap over $\mathcal{G}$:
			warm nodes (large~$q_i$) rely on their own features, cool nodes (small~$q_i$)
			are governed by graph topology.
			\textbf{(b)}~The shape of the learned polynomial~$p(\lambda)$ reveals how the graph topology influences the nodes with small $q_i$. For instance, an increasing profile penalizes high frequencies (favors homophily),
			a decreasing one penalizes low frequencies (favors heterophily).
			Together, $\{\{q_i\}_{i\in\mathcal{V}}, p(\cdot)\}$ \emph{is} the explanation.
		}
		\label{fig:explanation}
	\end{figure*}

	\begin{figure*}[t]
		\centering
		\begin{minipage}
			{0.42\textwidth}
			\small(\textbf{a) large $q_i$ everywhere}
			\centering
			\includegraphics[width=\textwidth]{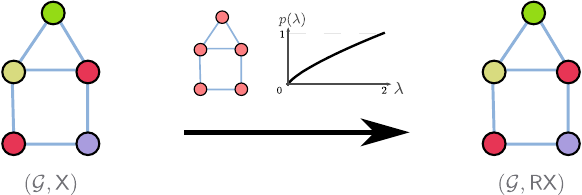}
		\end{minipage}\hfill
		\begin{minipage}
			{0.42\textwidth}
			\small(\textbf{b) mixed $q_i$}
			\centering
			\includegraphics[width=\textwidth]{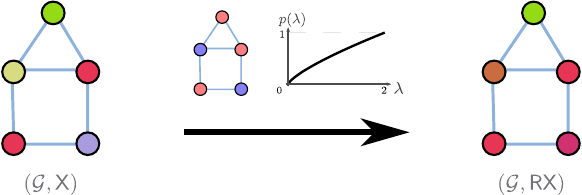}
		\end{minipage}
		\caption{%
			\textbf{Illustration of $\ma{R}\vec{x}$.} \textbf{(a)} large $q_i$ everywhere. $\ma{R}\approx\ma{I}$ regardless of $p(\lambda)$: the model ignores topology and behaves like an MLP.  \textbf{(b)} mixed $q_i$. Nodes with large $q_i$ keep their features. Nodes with small $q_i$ are influenced by their neighborhood (according to $p(\lambda)$). Here, with $p(\lambda)$ increasing, high frequencies are penalized, and homophily is encouraged.
		}
		\label{fig:illustration_Rx}
	\end{figure*}

	Crucially, our architecture consists of \emph{a single Tikhonov layer} of this 
	form. The final node features, e.g, those fed to the global pooler followed by the MLP, are simply obtained as $$\sigma\left(\ma{R}\ma{X}\ma{W}\right)=\sigma\left((p(\ma{L})+\ma{Q})^{-1}\ma{Q}\ma{X}\ma{W}\right).$$
	This deliberate simplicity is what makes the model interpretable by design: 
	the learned values $(q_i)_{i \in \mathcal{V}}$ provide a direct and faithful answer 
	to the research questions above. 
	By inspecting the $q_i$'s after training, one can 
	identify which nodes rely on their own features (those with large $q_i$), which nodes are governed by the graph topology (those with small $q_i$), and in such cases, how the topology affects the features (via the learned polynomial $p$). 
	One can, in particular, directly see whether the topology is exploited at all (if all $q_i\gg1$, $\ma{R}$ tends to the identity matrix and the Tikhonov layer reduces to an MLP). This is done \textit{without} any post-hoc approximation, since the explanation \emph{is} the model itself. 
	As such, we will refer to the pair $\{\{q_i\}_{i\in\mathcal{V}},p(\cdot)\}$ as the \emph{explanation}: see Figures~\ref{fig:explanation} and~\ref{fig:illustration_Rx} for illustrations. 
	
	Remarkably, expressivity is not sacrificed for the sake of interpretability.
	Rather than learning the $(q_i)_{i\in\mathcal{V}}$ as free parameters, they are
	produced from $(\mathcal{G},\ma{X})$ by a dedicated GNN — which we call the \emph{$\ma{Q}$-network} — that
	can be made arbitrarily deep and expressive. This design enables the model to naturally handle graphs of varying sizes, and conditions the scores $q_i$ on the graph's topology and node features. 
	This $\ma{Q}$-network concentrates the majority of the representational power of the model in a single, well-defined place. 
	The rest of the architecture (the Tikhonov layer and the final MLP) can thus keep its transparency and interpretability. 
	In this sense, our model follows a \emph{separation of concerns}: complexity lives in the $\ma{Q}$-network, interpretability lives in the Tikhonov layer. 
	See Fig.~\ref{fig:schema} for an illustration of the full architecture.

	\begin{figure*}[t]
		\begin{center}
			\includegraphics[width=\textwidth]{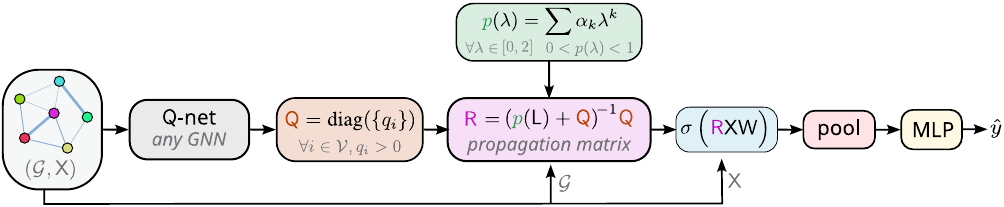}
		\end{center}
		\caption{%
			\textbf{Overview of the Tikhonov layer.} A $\ma{Q}$-network (any GNN)
			produces the node importances~$\ma{Q}=\text{diag}(\{q_i\})$; a single Tikhonov layer
			propagates features globally; pooling + an MLP head yield the prediction (here for graph classification). The learnable parameters are those of the $\ma{Q}$-network, the coefficients of the polynomial $p(\lambda)$, the weights $\ma{W}$, and the weights of the MLP head. 
		}
		\label{fig:schema}
	\end{figure*}

	\noindent\textbf{Additional properties.}
	Beyond interpretability, the Tikhonov layer naturally mitigates the effect of two notorious pathologies of deep 
	GNNs, namely \emph{oversmoothing} and \emph{oversquashing}.
	
	\textit{Oversmoothing.} Stacking many message-passing layers causes all node representations to collapse~\cite{li2018deeper}. Our architecture avoids involuntary oversmoothing in two ways. First, using a \emph{single} Tikhonov layer avoids depth-induced accumulation. Second, since $p(\lambda)>0$ on all of $[0,2]$, the operator $p(\ma{L})$ is invertible and $\ma{R}=(p(\ma{L})+\ma{Q})^{-1}\ma{Q}\to\mathbf{0}$ as all $q_i\to 0$: the only way to collapse node representations is to drive all importance scores to zero, which is a deliberate model choice rather than an unavoidable pathology. 
	
	\textit{Oversquashing.} In message-passing GNNs, long-range sensitivity $|\partial(\ma{R}\vec{x})_i/\partial x_j|$ decays exponentially with the hop-distance $d_{\mathcal{G}}(i,j)$ at a rate that is fixed by the architecture, regardless of model width~\cite{alon2021bottleneck}. In the Tikhonov layer the analogous quantity is $|\ma{R}_{ij}|$, which decays exponentially with $d_{\mathcal{G}}(i,j)$ at rate $2/(K\sqrt{\kappa})$ with $K$ the degree of $p(\cdot)$ and $\kappa=\operatorname{cond}_2(\ma{I}+\ma{Q}^{-1/2}p(\ma{L})\ma{Q}^{-1/2})$ (Proposition~\ref{prop:decay}). Crucially, $\kappa$ is \emph{learnable}: when the model drives some $q_i$'s small, $\kappa$ grows and $2/(K\sqrt{\kappa})\to 0$, making the decay negligible. The model thus learns to adjust its spatial reach end-to-end, trading local precision against long-range sensitivity through the $q_i$'s.
	
	Note that oversquashing and oversmoothing can occur in the Q-network as it is a standard potentially deep GNN: mitigation strategies from the literature should be applied to address them as needed.

	\section{Related works}\label{sec:related}
	
	\textbf{Spectral and Laplacian-based GNNs.}
	Early spectral GNNs~\cite{bruna2014spectral} defined convolutions via the eigendecomposition of $\ma{L}$; \citet{defferrard2016chebyshev} replaced this with Chebyshev polynomial approximations, and \citet{kipf2017semi} further simplified to a first-order filter. 
	All of these use polynomials of $\ma{L}$ as propagation operators. 
	Our $\ma{R}=(p(\ma{L})+\ma{Q})^{-1}\ma{Q}$ is instead the \emph{regularized inverse} of a polynomial of $\ma{L}$, corresponding to the closed-form solution of a Tikhonov problem. 
	\textsc{APPNP}~\cite{gasteiger2019predict} uses a personalized PageRank resolvent $\alpha(\ma{I}-(1-\alpha)\hat{\ma{A}})^{-1}$, structurally close to our formulation when $\ma{Q}=\alpha\ma{I}$, but with a fixed scalar $\alpha$; \textsc{GDC}~\cite{gasteiger2019diffusion} similarly uses fixed diffusion kernels. 
	The crucial distinction is our \emph{learnable, node-wise} $\ma{Q}=\mathrm{diag}(q_1,\ldots,q_n)$, which enables interpretability.
	ARMA graph convolutions~\cite{bianchi2021graph} also implement rational spectral filters; our homogeneous case is rational as well, but the learnable node-wise $\ma{Q}$ gives a variational and interpretable non-stationary extension.
	
	\noindent\textbf{Graph signal processing and Tikhonov regularization.}
	In graph signal processing~\cite{shuman2013emerging,ortega2022introduction}, the Laplacian quadratic form $\vec{z}^\top\ma{L}\vec{z}$ measures signal smoothness. 
	Graph Tikhonov regularization~\cite{zhou2004regularization,smola2003kernels} denoises a signal by minimizing $\|\vec{x}-\vec{z}\|^2_\ma{Q}+\vec{z}^\top\ma{L}\vec{z}$, whose solution is $(\ma{L}+\ma{Q})^{-1}\ma{Q}\vec{x}$, and has been studied in semi-supervised learning~\cite{zhu2003semi}. 
	To our knowledge, this has not been used as a \emph{learnable} GNN layer with a node-wise $\ma{Q}$. 
	When $\ma{Q}$ is heterogeneous with sparse large entries and $p$ penalizes high frequencies, the operator is also related to \emph{graph inpainting}~\cite{narang2013signal,chen_signal_2014}: nodes with large $q_i$ act as anchors, and the rest are interpolated by topology.
	
	\noindent\textbf{Attention and node importance in GNNs.}
	Graph attention networks~\cite{velickovic2018graph} learn edge-level importance weights. 
	Our $q_i$ plays an analogous but distinct role: it governs the balance between node $i$'s own features and topological pressure at the \emph{node} level, through a variational formulation. 
	Unlike an attention coefficient, which only affects its two endpoints, a single $q_i$ influences the solution at every other node through the resolvent $(p(\ma{L})+\ma{Q})^{-1}$, giving it a global footprint.
	
	\noindent\textbf{Self-explainable and interpretable GNNs.}
	Our model is a \emph{self-explainable} GNN (SE-GNN): interpretability is built into the architecture, not recovered post-hoc via methods such as \textsc{GNNExplainer}~\cite{ying2019gnnexplainer}, \textsc{PGExplainer}~\cite{luo2020parameterized}, or gradient saliency~\cite{pope2019explainability}, which produce explanations without faithfulness guarantees (see~\citet{longa2025explaining} for a survey; \citet{azzolin2025reconsidering} for a critical analysis).
	Typical SE-GNNs such as \textsc{SE-GNN}~\cite{dai2021towards} or \textsc{ProtGNN}~\cite{zhang2022protgnn} follow a two-module design $g=f\circ q$: an explanation extractor $q$ identifies a subgraph, and a classifier $f$ maps it to a prediction.  
	This factorization is critically vulnerable: $f$ can achieve optimal risk while the explanation produced by $q$ is provably unrelated to the model's actual computations~\cite{azzolin2026degenerate}.
	Our architecture avoids this two-module failure mode since there is no detached explanatory subgraph passed to a classifier as an arbitrary code. 
	The prediction pipeline is a \emph{single computation} $\text{MLP}(\text{pool}(\sigma((p(\ma{L})+\ma{Q})^{-1}\ma{Q}\,\ma{X}\ma{W})))$, where $\ma{Q}$ and $p$ are constitutive parts of the propagation operator $\ma{R}$. 
	Proposition~\ref{prop:structural_faithfulness} formalizes this structural coupling: for fixed $p$, the map $\ma{Q}\mapsto\ma{R}$ is injective, so changing the node-importance explanation necessarily changes the propagation operator. 
	While the $\ma{Q}$-network may still learn shortcut or label-encoding patterns, the explanation is faithful to the Tikhonov propagation induced by the model. 
	Also, rather than selecting a discrete subgraph, the Tikhonov layer assigns a continuous scalar importance $q_i$ to each node, directly quantifying the local balance between features and topology: a graded, spatially resolved answer better suited to interpretability than a binary selection.

	\section{Methodology}\label{sec:contribution}
	In its general form, the multi-channel Tikhonov layer we propose can be written as a $J$-channel layer: $\sigma\left(\ma{R}_1\ma{X}\ma{W} | \cdots | \ma{R}_J\ma{X}\ma{W}\right)$, 
	where (i) the notation $\left(\ma{M}_1 | \cdots | \ma{M}_J\right)$ represents column-wise concatenation of the matrices $\{\ma{M}_j\}$; (ii) $\ma{R}_j$ reads  $\ma{R}_j=\left(p_j(\ma{L})+\ma{Q}_j\right)^{-1}\ma{Q}_j$ where $p_j(\cdot)$ are learnable polynomials verifying $0< p_j(\lambda)< 1$ on $[0,2]$. Each channel enjoys the variational interpretation of Eq.~\eqref{eq:VI}; (iii)~The $\ma{Q}_j$'s are diagonal matrices of positive node-importance parameters. Each $\ma{Q}_j$ is built via a dedicated $\ma{Q}$-network: there are thus $J$ independent $\ma{Q}$-networks; (iv) there is only one learnable matrix weight $\ma{W}$. This uniqueness enables identifiability even when $J>1$ (see Proposition~\ref{prop:multichannel_id}). 
	
	Each channel $j$ learns a global graph filter $p_j(\lambda)$ along with where on the graph and with what intensity it should be applied on the features, via the diagonal entries of  $\ma{Q}_j$. If the data exhibits both homophily and heterophily at different locations and/or scales, the multi-channel design enables the model to learn different local behaviors. 
	This is similar to the ``multi-head attention'' in transformers, where each head can learn different patterns. In our experimental section, we will set $J=1$. 
	
	\subsection{Theoretical properties}
	
	We establish the fundamental properties of the Tikhonov layer. Unless stated otherwise, all results are for the single-channel case (extensions to multi-channel are immediate); proofs are in appendix~\ref{app:proofs}. Let us recall our notation: $\ma L$ is the normalized Laplacian with eigenvalues $\{\lambda_k\}$, $p(\cdot)$ a polynomial of order $K>0$ verifying $0<p(\lambda)<1$ for $\lambda\in[0,2]$, $\ma Q$ a diagonal matrix of strictly positive entries,  $\ma{R}=(p(\ma{L})+\ma{Q})^{-1}\ma{Q}$ the Tikhonov operator, and $d_{\mathcal{G}}(i,j)$ the hop distance between $i$ and $j$ in $\mathcal{G}$. 
	
	\begin{definition}[Complete $K$-hop support]
		Let $\mathbf{P} = p(\mathbf{L})$. It is well-known that $\mathbf{P}$ is $K$-localized: $d_{\mathcal{G}}(i,j) > K \implies P_{ij} = 0$. We say that $\mathbf{P}$ has \emph{complete $K$-hop support} if the converse also holds:
		\[
		P_{ij} \neq 0 \iff d_{\mathcal{G}}(i,j) \le K.
		\]
	\end{definition}
	Assuming that $\ma P$ has complete $K$-hop support is a very mild condition. For instance, if the coefficients of $p$ are drawn at random from a continuous distribution—as is typically done during initialization before training—then $\ma P$ has complete $K$-hop support almost surely; see Lemma~\ref{lem:Khop} for a precise statement. Moreover, we only use this assumption to simplify claims (iii) of Propositions~\ref{prop:properties} and~\ref{prop:receptive}.
	\begin{proposition}[Basic properties of $\ma{R}$]\label{prop:properties}
		(i)~The inverse of $p(\ma{L})+\ma{Q}$ exists: $\ma{R}$ is well-defined. (ii)~The spectrum of $\ma{R}$ is contained in $(0,1)$; in particular, $\ma{R}$ is injective. (iii) Under complete $K$-hop support, $\ma{R}$ is symmetric if and only if the $q_i$'s are constant on each connected component of $\mathcal{G}$.
	\end{proposition}
	Point (iii) means that heterogeneous $q_i$'s make the influence of node~$j$ on node~$i$ differ from that of node~$i$ on node~$j$: this asymmetric coupling lets the model represent direction-dependent influence.
	\begin{proposition}[Effect of $q_i$ on feature preservation and local $p(L)$-harmonicity]\label{rem:regimes}\label{prop:harmonicity}
		Let $p_{\text{max}}=\max_k p(\lambda_k)$. Set $\vec{z}=\ma{R}\vec{x}$. For every $i\in\mathcal{V}$:
		\begin{equation*}
			|z_i - x_i|\leq \frac{p_{\text{max}}}{q_i}\,\|\vec{x}\|_2,
			\qquad
			|[p(\ma{L})\vec{z}]_i| \;\leq\; \sqrt{q_i}\,\sqrt{p_{\text{max}}}\,\left\|\vec{x}\right\|_2.
		\end{equation*}
	\end{proposition}
	Thus, when $q_i\gg p_{\text{max}}$, node $i$'s features are preserved.
	Conversely, when $q_i\ll p_{\text{max}}$, node~$i$'s output is nearly $p(\ma{L})$-harmonic: its value 
	is dictated by the graph topology rather than by its own feature.
	%
	%
	\begin{proposition}[Spectral density and strict generalization]\label{prop:density}
		The family of homogeneous Tikhonov operators $\{(p(\ma{L})+q\ma{I})^{-1}q\ma{I}\}$ with $q>0$ and $0<p(\lambda)<1$ on $[0,2]$ is dense in the uniform norm in the space of continuous spectral filters $\{g(\ma{L}) \mid g:[0,2]\to(0,1)\}$. Furthermore, $\ma{R} = (p(\ma{L})+\ma{Q})^{-1}\ma{Q}$ corresponds to a spectral filter $g(\ma{L})$ if and only if $\ma{Q}$ is homogeneous ($\ma{Q}=q\ma{I}$). For heterogeneous $\ma{Q}$, $\ma{R}$ generically does not commute with $\ma{L}$, strictly extending the expressivity of spectral graph filters.
	\end{proposition}
	In the homogeneous case ($\ma{Q}=q\ma{I}$), the Tikhonov operator's rational parameterization matches the expressivity of polynomial spectral GNNs. A heterogeneous $\ma{Q}$ breaks the stationarity assumption of traditional spectral methods, enabling flexible, highly expressive \emph{node-adaptive filtering}.
	\begin{proposition}[Global receptive field]\label{prop:receptive} 
		(i)\;$\ma{R}_{ij}=0$ if $i,j$ belong to different connected components. (ii)\;If $p(\lambda)=a+b\lambda$ with $b>0$, then $\ma{R}_{ij} > 0$ for every $(i,j)$ in
		the same connected component. (iii)\;More generally for any complete $K$-hop support $p$: for almost every $\ma{Q}$ (\emph{i.e.}, for all $\ma Q$ outside a measure-zero exceptional set): $R_{ij}\neq 0$ for every  $(i,j)$ in the same connected component.
	\end{proposition}
	Thus, generically in $\ma{Q}$, the Tikhonov operator $\ma{R}$ has a global receptive field. 
	Increasing the polynomial degree $K$ enriches the set of representable spectral filters but does \emph{not}  enlarge the spatial reach, which is already global for $K=1$. 
	While non-zero entries are good news to mitigate oversquashing, they alone do not prevent it: the magnitude of $\ma{R}_{ij}$, quantified by the next proposition, is what truly matters.
	\begin{proposition}[Spatial decay of $\ma{R}$]\label{prop:decay}
		Let the graph be connected and $\kappa$ be the condition number of $\,\ma{I}+\ma{Q}^{-1/2}p(\ma{L})\,\ma{Q}^{-1/2}$. 
		Then, for any pair $(i,j)$:
		\begin{equation*}
			|\ma{R}_{ij}| \;\leq\; 
			2\sqrt{\frac{q_j}{q_i}}\;
			\exp\!\left(-\frac{2}{K\sqrt{\kappa}}d_{\mathcal{G}}(i,j)\right).
		\end{equation*}
	\end{proposition}
	This bound makes the oversquashing behavior precise. Both the prefactor 
	$2\sqrt{q_j/q_i}$ and the decay rate $2/(K\sqrt{\kappa})$ depend on the learned $\ma{Q}$. When all $q_i$ are large, $\kappa \approx 1$ and the decay is rapid, making the layer act locally, similarly to a standard spectral filter. Conversely, when some $q_i$ entries are small, long-range sensitivity is boosted by two coupled effects: (i) the prefactor $\sqrt{q_j/q_i}$ amplifies distant contributions at destination nodes with small $q_i$; (ii) $\kappa$ increases, weakening the exponential attenuation across the graph. Oversquashing is thus not absent by default, but \emph{controllable by design}: the model learns the $q_i$'s end-to-end and can allocate long-range sensitivity where needed.

	\begin{proposition}[Structural coupling of $\ma{Q}$ and \ma{R}]\label{prop:structural_faithfulness}
		The map 
		$\ma{Q}\mapsto(p(\ma{L})+\ma{Q})^{-1}\ma{Q}$ is injective on 
		the set of diagonal matrices with strictly positive entries. In particular, 
		two models sharing the same $p$, $\ma{W}$, and MLP head but differing in 
		their $\ma{Q}\neq\ma{Q}'$ produce distinct propagation 
		operators.
	\end{proposition}
	Unlike in two-module SE-GNN architectures, where a classifier can treat the extracted explanation as a separate label code, here $\ma{Q}$ directly defines the propagation operator seen by the MLP head. 
	Thus, the explanation cannot be changed without changing the Tikhonov computation, which guarantees faithfulness of $\ma{Q}$ to the induced propagation operator.
	
	\begin{proposition}[Joint identifiability of $(p,\ma{Q})$]\label{prop:identifiability}
		Let $(p_1,\ma{Q}_1)$ and $(p_2,\ma{Q}_2)$ be two pairs with diagonal $\ma{Q}_j\succ 0$ and polynomials $0< p_j(\lambda)< 1$ on $[0,2]$. 
		If both pairs yield the same operator, 
		$(p_1(\ma{L})+\ma{Q}_1)^{-1}\ma{Q}_1 = (p_2(\ma{L})+\ma{Q}_2)^{-1}\ma{Q}_2$, 
		then there exists $\alpha>0$ such that $\ma{Q}_2=\alpha\,\ma{Q}_1$ and $p_2(\ma{L})=\alpha\,p_1(\ma{L})$.
	\end{proposition}

	The only operator-level non-identifiability is a global rescaling: $(p(\ma{L}),\ma{Q})$ and $(\alpha p(\ma{L}),\alpha\ma{Q})$ produce the same $\ma{R}$ since the factor cancels in $(p(\ma{L})+\ma{Q})^{-1}\ma{Q}$. 
	This is benign: the \emph{relative} pattern of the $q_i$'s, driving interpretation, is invariant to such rescaling.

	\begin{proposition}[Multi-channel identifiability]\label{prop:multichannel_id}
		Let $\ma{H}=\ma{X}\ma{W}$ have rank $n$.
		Consider the pre-activation $J$-channel layer output $\left(\ma{R}_1\ma{H}\mid\cdots\mid\ma{R}_J\ma{H}\right)$ with $\ma{R}_j=(p_j(\ma{L})+\ma{Q}_j)^{-1}\ma{Q}_j$. 
		If two parameter sets $\{(p_j,\ma{Q}_j)\}_{j=1}^J$ and 
		$\{(p'_j,\ma{Q}'_j)\}_{j=1}^J$ yield the same output up to a permutation of the channel blocks, then there exists a permutation $\pi$ of $\{1,\ldots,J\}$ and 
		positive scalars $\alpha_j$ such that, for every $j$,
		$p'_j(\ma{L}) = \alpha_{j}\, p_{\pi(j)}(\ma{L})$ and 
		$\ma{Q}'_j = \alpha_{j}\, \ma{Q}_{\pi(j)}$. 
		In other words, up to the usual finite-spectrum ambiguity for polynomials, the only non-identifiabilities are per-channel rescaling 
		and channel permutation.
	\end{proposition}

	If $J$ independent weight matrices $\{\ma{W}_j\}$ are used instead of the unique $\ma{W}$ above, a 
	cross-channel compensation becomes possible: one channel's $\ma{W}_j$ can 
	absorb work from another, breaking per-channel identifiability. Sharing 
	$\ma{W}$ across channels eliminates this issue at a modest cost in 
	expressivity, and is recommended in our framework, where interpretability is a priority.

	\subsection{Practical implementation details}	
	
	\noindent\textbf{Matrix inversion.} Computing $\ma{RX}=\left(p(\ma{L})+\ma{Q}\right)^{-1}\ma{Q}\ma{X}$ for a feature matrix $\ma{X}\in\mathbb{R}^{n\times d}$ requires solving $d$ linear systems with the symmetric positive definite matrix $\ma{M}=p(\ma{L})+\ma{Q}$. An exact sparse Cholesky factorization costs up to $\mathcal{O}\left(dn^2+(Km)^{3/2}\right)$, where $m$ is the number of edges and $K=\deg(p)$. While cheaper than dense $\mathcal{O}(n^3)$ inversion, this is significantly more expensive than standard GNN layers (\eg, GIN or GCN) operating in $\mathcal{O}(dm)$. We therefore use Preconditioned Conjugate Gradient (PCG). 
	PCG solves $\ma{M}\vec{z}=\ma{Q}\vec{x}$ iteratively using matrix-vector products with $\ma{M}$ at $O(Km)$ per iteration. The number of iterations before convergence scales as $\mathcal{O}(\sqrt{\kappa(\ma{M})})$, where $\kappa(\ma{M})$ is the condition number. With a diagonal (Jacobi) preconditioner $\ma{P}=\mathrm{diag}(\ma{M})$, which is trivial to compute, the condition number is significantly reduced, typically yielding convergence in a small number of iterations. For $d$ right-hand sides, the total cost is $O(dTKm)$, where $T$ is the number of PCG iterations, yielding much better scalability than sparse Cholesky on large graphs.
	
	\begin{remark}[Effective receptive field of PCG]\label{rem:pcg-reach}
		With no preconditioner or a diagonal preconditioner such as Jacobi, each solver iteration applies node-wise rescaling and multiplication by $\ma{M}$. Since $\ma{M}$ mixes information across at most $K$ hops, after $T$ iterations the approximation's effective reach is at most $(T-1)K$ hops. Thus, while the exact Tikhonov operator is global (Proposition~\ref{prop:receptive}), a truncated PCG solve acts as a local polynomial approximation unless $T$ is large. This reflects a solver limitation, not a message-passing bottleneck: each iterate remains linear in $\vec{x}$ without nonlinear compression. The true long-range sensitivity is governed by the learned $\ma{Q}$ (Proposition~\ref{prop:decay}); $T$ merely controls approximation accuracy. For tasks requiring long-range interactions, $T$ should be set to large values (typically on the order of the graph diameter).
	\end{remark}

	To eliminate memory overhead and numerical instability, we use implicit differentiation rather than backpropagating through the solver iterations. 
	Let $\vec{z}=\ma{M}^{-1}\ma{Q}\vec{x}$ be the forward output and $\bar{\vec{z}}=\partial\ell/\partial\vec{z}$ the incoming gradient from the loss $\ell$. 
	Differentiating $\ma{M}\vec{z}=\ma{Q}\vec{x}$ with respect to $q_i$ (for instance) gives $\partial_{q_i}\ma{M}\vec{z}+\ma{M}\partial_{q_i}\vec{z}=\partial_{q_i}\ma{Q}\vec{x}$, \ie: $\partial_{q_i}\vec{z} = \ma{M}^{-1} \partial_{q_i}\ma{Q} (\vec x - \vec z)$. Applying the chain rule:
	\begin{equation*}
		\frac{\partial\ell}{\partial q_i} = \bar{\vec{z}}^\top\ma{M}^{-1} \partial_{q_i}\ma{Q} (\vec x - \vec z) = u_i(x_i-z_i)
	\end{equation*}
	where $\vec{u}=\ma{M}^{-1}\bar{\vec{z}}$ is the adjoint variable. Computing $\vec{u}$ is the only expensive step and requires one additional PCG solve with the same matrix $\ma{M}$. 
	Gradients with respect to $\vec{x}$ and polynomial coefficients can also be expressed via $\vec{u}$. 
	Thus, we run PCG in the forward pass without storing the computational graph, and use this single adjoint solve for all gradients in the backward pass.

	\noindent\textbf{Bounding the $q_i$'s.} The $\ma{Q}$-network produces unconstrained scalar outputs $\tilde{q}_i\in\mathbb{R}$ for each node $i$. To map these to strictly positive values while ensuring numerical stability, we apply $q_i = \exp\!\big(\min(\tilde{q}_i,\, \log q_{\max})\big) + q_{\min}$ 
	which clamps in log-space to prevent overflow in the exponential and adds a floor $q_{\min}$ to guarantee strict positivity. Despite this clamping, the condition number can still be substantial, making preconditioning essential for efficient CG convergence.

	\noindent\textbf{Enforcing $0< p(\cdot)< 1$ on $[0,2]$.} This constraint is enforced via Bernstein polynomials, writing $p(\lambda)=\sum_{k=0}^{K} \operatorname{sigmoid}(\theta_{k})\,b_{k,K}\!\left(\tfrac{\lambda}{2}\right)$ with $\{\theta_{k}\}$ real learnable parameters  and Bernstein basis functions $b_{k,K}(t)=\binom{K}{k}t^k(1-t)^{K-k}$ that are all positive for $t\in[0,1]$. Since $\operatorname{sigmoid}(\theta_{k})\in(0,1)$ and $\sum_k b_{k,K}(t)=1$ on $t\in[0,1]$, this parametrization naturally enforces $0<p(\lambda)<1$ for all $\lambda\in[0,2]$. 
	
	\begin{figure}[t]
		\centering
		\includegraphics[width=\textwidth]{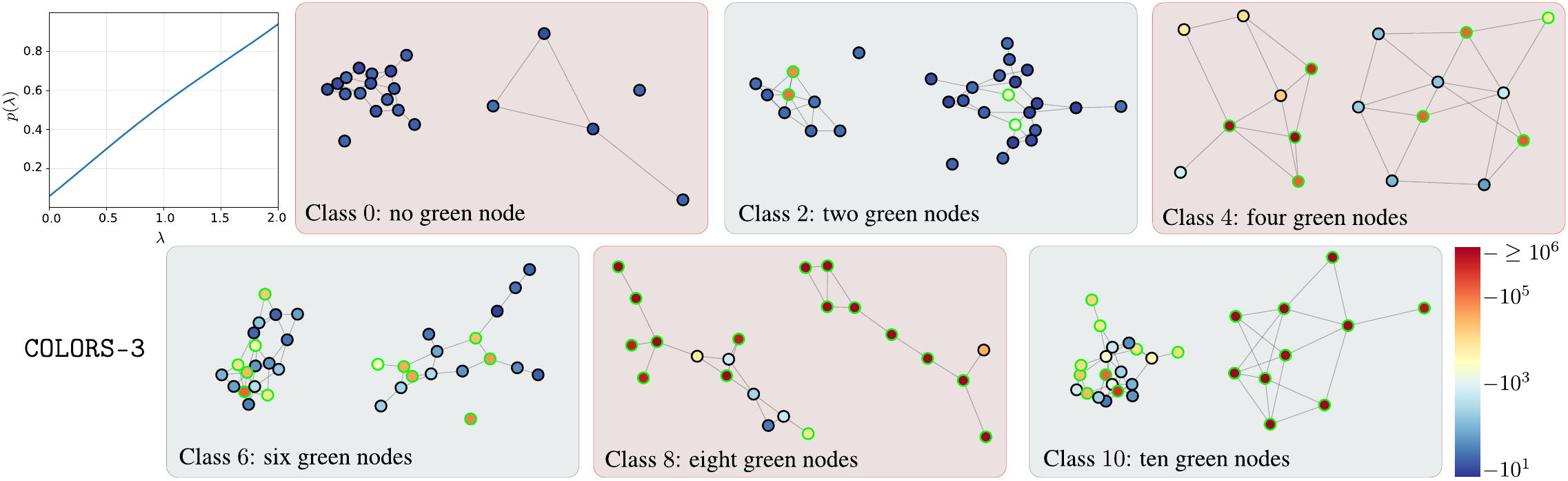}
		\caption{\textbf{Colors-3.} 
			Node color encodes the learned $q_i$ (log scale; blue: $q_i= 10$, red: $q_i\geq 10^6$); green circles mark ground truth. Top-left: learned $p(\lambda)$. 
			Best run (in terms of validation loss) out of 5 runs (with different seeds), which all produced a test accuracy of $100\%$: perfect classification.
		}
		\label{fig:colors3}
	\end{figure}
	
	\begin{figure}[t]
		\centering
		\includegraphics[width=\textwidth]{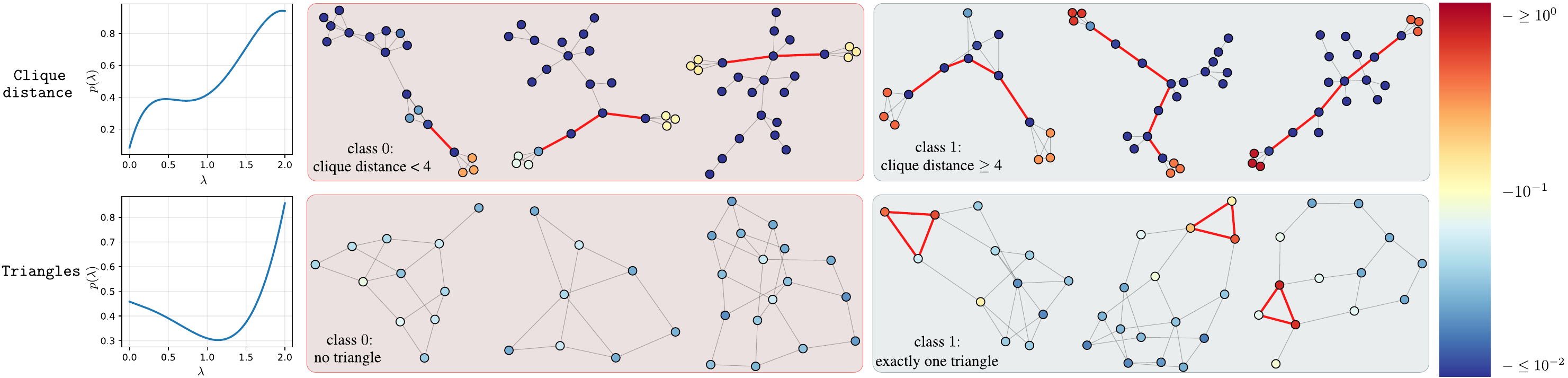}
		\caption{\textbf{Clique distance (top) and Triangles (bottom) datasets.} 
			Node color encodes the learned $q_i$ (log scale; blue: $q_i\leq 10^{-2}$, red: $q_i\geq 10^{0}$); red edges highlight ground-truth (triangle edges; shortest path between cliques). Left: learned polynomials $p$. 
			Best (in terms of validation loss) of 5 runs. 
			For \texttt{Clique distance}, the average test accuracy is $99.9\%$ (the run shown has test accuracy of $100\%$). For \texttt{Triangles}, the average test accuracy is $95\%$ (the run shown has test accuracy of $99.4\%$). }
		\label{fig:triangles}
	\end{figure}
	
	\section{Experiments}
	\label{sec:experiments}
	
	We evaluate the Tikhonov layer along two axes: \emph{interpretability}
	(does the learned explanation $\{p,\{q_i\}\}$ reveal meaningful structure?)
	and \emph{long-range propagation} (does the global receptive field help on
	tasks that require it?). Choice of hyperparameters is discussed in Appendix~\ref{app:hp}.
	
	\subsection{Interpretability on controlled tasks}
	\label{sec:exp-interpretability}
	
	We first study a few synthetic datasets for which the ground-truth role
	of features vs.\ topology is known by construction, so that the
	learned explanation can be objectively assessed.

	\noindent\textbf{Colors-3~\cite{knyazev2019understanding}.}
	This dataset requires the model to count green nodes in small graphs (we only consider the relevant first subset of the full dataset). The task is solvable by node features alone, but standard message-passing GNNs can be biased by topology. 
	The Tikhonov layer reaches a test accuracy of 100\% and, as shown in Figure~\ref{fig:colors3}, learns large $q_i$'s everywhere ($\gg 1$), effectively ignoring the graph structure and nearly reducing to an MLP. More strikingly here in the plotted explanation \footnote{This is striking because the network does not need to do this. Indeed, other explanations found very large $q$-values ($\geq 10^3$) everywhere, not necessarily correlated with the green nodes. But large $q$-values everywhere just means the layer reduces to an MLP - and an MLP solves the task perfectly!}, it assigns almost systematically higher $q_i$'s to green nodes than to non-green ones: the explanation $\{q_i\}$ automatically singles out the nodes that carry the task-relevant feature. 
	
	\noindent\textbf{Clique distance~\cite{tolmachev2021bermuda}.}
	Each graph contains two cliques connected by a path of varying length. The task is to predict whether the inter-clique distance is below~4 (class~0) or at least~4 (class~1). Node features are uninformative (all-ones): the task is purely structural. As shown in Figure~\ref{fig:triangles} (top row), the model learns an increasing $p(\lambda)$ and assigns very low $q_i$'s to the path nodes connecting the cliques, while clique nodes are assigned much larger $q_i$'s. The model thus exploits the topology of the connecting path to sense its length, with the low-$q_i$ path nodes acting as topology-sensitive relays.
	The Tikhonov layer achieves a test accuracy of 100\% (\cite{tolmachev2021bermuda} reports GIN achieving 97\%)  and provides a meaningful explanation: cliques serve as anchors and diffusion propagates between them.

	\noindent\textbf{Triangles~\cite{tolmachev2021bermuda}.}
	This dataset asks whether a graph contains exactly one triangle (class~1) or none (class~0). 
	Node features are uninformative (all-ones), making the task purely structural. 
	As shown in Figure~\ref{fig:triangles} (bottom row), the model learns a mostly low-pass filter $p(\lambda)$ and small $q_i$'s almost everywhere, hence relies primarily on topology.  
	In class-1 graphs, however, triangle vertices receive noticeably larger $q_i$'s, so the explanation pinpoints the motif without supervision. 
	The model reaches $99.4\%$ test accuracy, outperforming the best GAT baseline reported~\cite{tolmachev2021bermuda} for this dataset ($94\%$, using degree and node-ID one-hot features) and nearly matching the Transformer accuracy ($100\%$), while also revealing that the prediction is driven by identifying the triangle nodes.
	
	\begin{figure}[t]
		\centering
		\includegraphics[width=\textwidth]{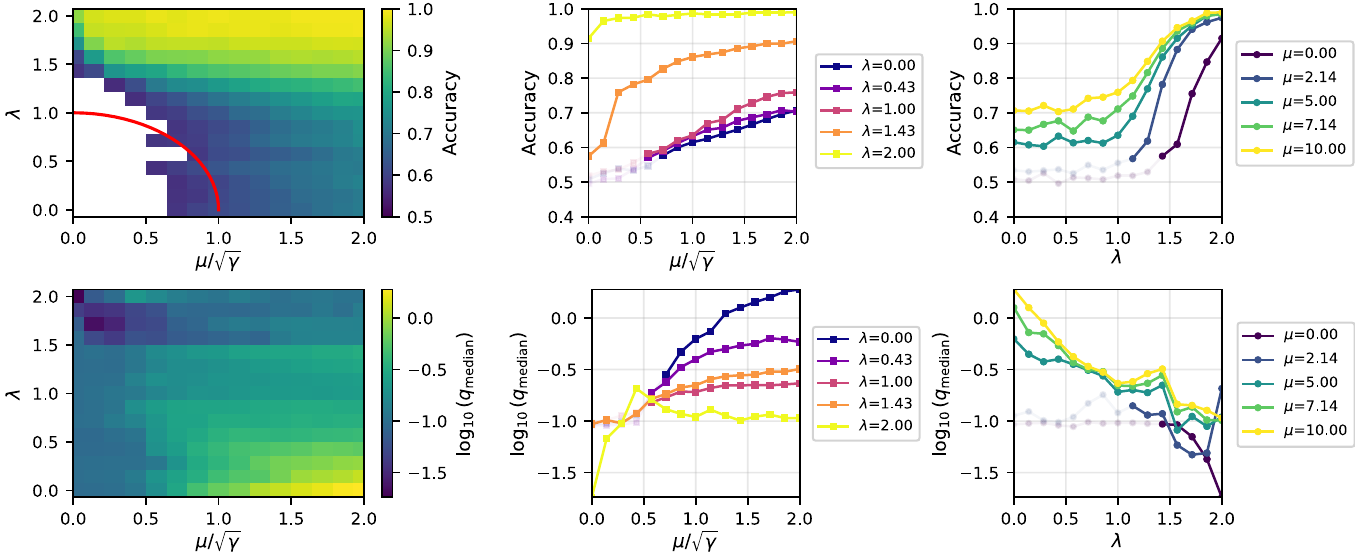}
		\caption{\textbf{CSBM dataset.} \textbf{(Top) test accuracy}. Left: accuracy on the parameter-space $(\lambda,\mu/\sqrt{\gamma})$. The red quarter-circle is the theoretical threshold at large $n$. When the accuracy $<$ 0.56, it is left blank. Two right-top plots: same for some choices of $\lambda$ and $\mu/\sqrt{\gamma}$. \textbf{(Bottom) median of $q$-values} for all nodes of all graphs (from both classes). Left: median of $q$-values on the parameter-space $(\lambda,\mu/\sqrt{\gamma})$. Two right-top plots: same for some choices of $\lambda$ and $\mu/\sqrt{\gamma}$. For the four right-most plots: when accuracy is below 0.56, transparency is used. All plots are averaged over 6 runs. Learned polynomials are not plotted: they all converge close to $\lambda/2$ throughout the parameter space.
		}
		\label{fig:csbm}
	\end{figure}
	
	\noindent\textbf{CSBM (Contextual Stochastic Block Model).}
	The CSBM~\cite{deshpande_contextual_2018} is a random graph model with node features, where topology and features are correlated with the same latent communities. Since it admits a known asymptotic recovery threshold, it provides a controlled setting to evaluate whether the learned explanation $\{q_i\}$ correctly identifies the source of the predictive signal (features, topology, or both); see Appendix~\ref{sec:csbm_background} for details. The topology is sampled from a stochastic block model with signal strength $\lambda$, where edges are more likely within than across communities. Weak recovery from topology alone is possible iff $\lambda>1$. Node features are drawn from a Gaussian mixture aligned with the labels, with signal strength $\mu/\sqrt{\gamma}$; weak recovery from features alone is possible iff $\mu/\sqrt{\gamma}>1$. Joint weak recovery is possible~\cite{deshpande_contextual_2018} iff $\lambda^2+\mu^2/\gamma>1$. 
	We introduce the following classification task: distinguish connected CSBM graphs (class~1) from a null model (class~0). For class~0, we sample a CSBM graph with the same $\lambda$, then randomly rewire its edges while preserving degree distribution and connectivity, removing community structure; features contain no signal ($\mu=0$). We introduce this dataset to evaluate our method in a controlled yet challenging setting. For each $(\lambda,\mu/\sqrt{\gamma})$, we generate 2000 graphs (1000 per class) and use a 10-fold split.
	Graphs have $n=100$ nodes split into two equal communities and average degree $d=10$. We fix $\gamma=25$.

	As shown in Figure~\ref{fig:csbm}, 
	the model correctly tracks the theoretical threshold\footnote{The theoretical threshold is sharp (it is a phase transition) for $n$ and $p$ (the node features' dimension) growing to infinity with $\gamma=n/p$ constant. For finite $n$ and $p$, slightly correlated detections inside the red quarter-circle are to be expected.}: accuracy is high above the curve $\lambda^2+\mu^2/\gamma=1$ and drops to chance below it. The region at small $\mu$ and just over the threshold $\lambda>1$ is difficult to recover. 
	In fact, this region is well-known to be challenging, see~\cite{massoulie_community_2014,decelle_asymptotic_2011}. 
	Of particular interest for us here, the median of the learned $q$ values (bottom row) transitions smoothly from large values (feature-dominated regime, large $\mu/\sqrt{\gamma}$) to small values (topology-dominated regime, large $\lambda$), with an intermediate range in the mixed regime. The explanation faithfully detects the underlying signal structure, with no direct supervision on which source to use.

	\subsection{Resistance to oversquashing}
	\label{sec:exp-oversquashing}
	
	We evaluate on the ECHO-diam task from the ECHO benchmark \cite{miglior2026can}: it is a graph regression task and consists in predicting the diameter of a synthetic graph. This requires capturing the longest shortest path, which is inherently a global property. Results are shown in Figure~\ref{fig:echo}. The learned filter is low-pass, and the $\ma{Q}$-network assigns large values to peripheral nodes (leaf-like or degree-one endpoints) and small values to interior nodes. Interior nodes with small $q$ undergo heavy spectral smoothing and carry a nearly uniform signal; peripheral nodes with large $q$ effectively anchor this smooth field with their own features. Global pooling aggregates this contrast into a prediction: the more spread-out the periphery, the higher the diameter. Intuitively, the model identifies the endpoints of the longest path and uses their amplitude contrast with the interior as a proxy for the number of nodes lying between them.
	The test MAE is $1.17$, between the best deep GNN ($\sim 1.6$) and the best transformer ($1.05$), as reported in~\cite{miglior2026can}, with the added benefit of a built-in explanation of how the regression is performed.

	\begin{figure}[t]
		\centering
		\includegraphics[width=\textwidth]{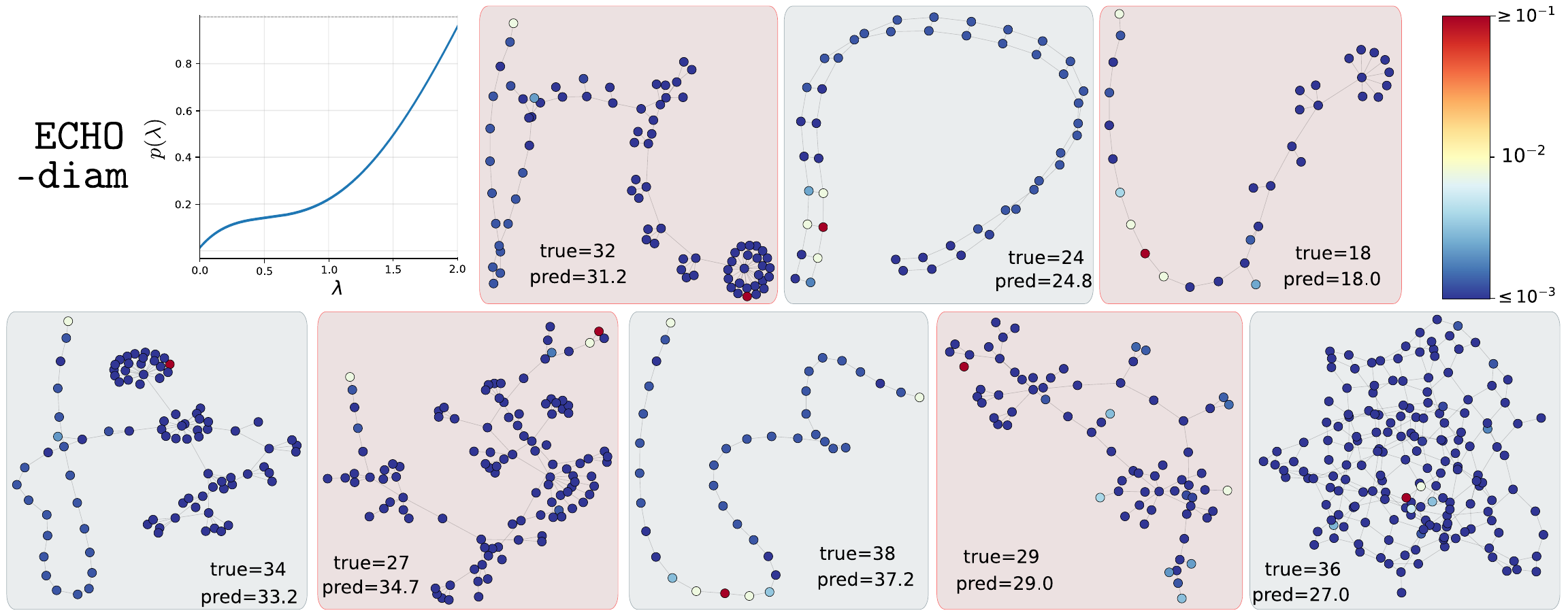}
		\caption{\textbf{ECHO-diam benchmark.} The learned polynomial along with the learned $q$-values plotted on 8 example graphs (log scale; blue: $q_i\leq 10^{-3}$, red: $q_i\geq 10^{-1}$). The true diameter is written, along with the prediction of the Tikhonov layer. 
			Best (in terms of validation loss) out of 5 runs. The average test MAE is $1.20$ (the run shown has a test MAE of $1.17$). The MAE is the average error of prediction of the diameter. }
		\label{fig:echo}
	\end{figure}
	
	\section{Concluding remarks}\label{sec:conclusion}
	We propose an original take on the problem of interpretability in GNNs. By severely constraining the form of the layer, we force the network to learn by-products (the node importance $q$-values and the graph filter $p(\lambda)$) that give us insight into how the network is solving the task. Interpretation is a difficult endeavour. Our method encounters two common drawbacks of the literature: i) some explanations can still be hard to understand for us humans (see, for instance, the second explanation for ECHO-diam discussed in Appendix~\ref{add_expe}), ii) some tasks have many different explanations (see, for instance, the second explanation for Triangles discussed in Appendix~\ref{add_expe}). Further research could be to add an auxiliary loss to favor some explanations rather than others. A limitation of the Tikhonov layer is its computational cost in cases where the task requires PCG to approximate long-distance interactions encoded in $\ma{R}$ (as one needs to increase the number of PCG iterations). Promising solutions come from random numerical linear algebra that are able to obtain Monte-Carlo approximations of $\ma{R}\vec{x}$ (for now only in the case where $p(\ma{L})=\ma{L}$), without computing $\ma{R}$, via carefully designed random walks on the graph~\cite{barthelme2025spectrum}, for a total cost linear in $n$.  
	
	\begin{ack}
		This work was partially funded by the European Union (Marie Curie
		project Rand4TrustPool, 101148828), the R\'egion Auvergne-Rhône-Alpes project
		TrustGNN, the ANR project GRANOLA (ANR-21-CE48-0009), and the Norwegian Research Council project 345017: \emph{RELAY: Relational Deep Learning for Energy Analytics}. We gratefully thank NVIDIA Corporation for donating the GPUs used in this project.
	\end{ack}

	\medskip
	
	{\small
		\bibliography{biblio.bib}
	}

	
	\appendix

	\section{Proofs}
	\label{app:proofs}
	
	\begin{lemma}\label{lem:Khop}
		For a fixed graph $\mathcal G$, $\ma P=p(\ma{L})$ has complete $K$-hop support for generic polynomial coefficients. 
		Equivalently, the set of coefficients for which the property fails
		has Lebesgue measure zero.
	\end{lemma}
	\begin{proof}
		Write $\ma L = \ma I-\tilde{\ma A}$ with $\tilde{\ma A}=\ma D^{-1/2}\ma A\ma D^{-1/2}$. \textit{Step 1:} fix two nodes $i\neq j$ and denote by  $0<d_{\mathcal G}(i,j)=m\le K$ their hop distance.
		Since no walk of length $\ell<m$ connects $i$ to $j$, $(\tilde{\ma A}^\ell)_{ij}=0$ for all $\ell<m$.
		The binomial theorem then boils down to
		\[
		(\ma L^m)_{ij} = \sum_{\ell=0}^m \binom{m}{l}(-1)^\ell(\tilde{\ma A}^\ell)_{ij} = (-1)^m(\tilde{\ma A}^m)_{ij} \neq 0
		\]
		as $\tilde{\ma A}^m$ counts weighted walks of length $m$, of which at least one exists between $i$ and $j$.
		Thus $\ma{P}_{ij}(\vec\alpha)=\sum_{k=0}^K\alpha_k(\ma L^k)_{ij}$ is a non-trivial linear form in $\vec\alpha$,
		so its zero set is a hyperplane in $\mathbb R^{K+1}$, hence of measure zero. \textit{Step 2:} If $i=j$, hence $m=0$, one has: $(\ma L^m)_{ii}=1\neq 0$ and $\ma{P}_{ii}(\vec\alpha)=\sum_{k=0}^K\alpha_k(\ma L^k)_{ii}$ is also a non-trivial linear form in $\vec\alpha$: its zero set is also of measure zero. \textit{Step 3:} 
		The conclusion follows by taking a finite union over all pairs $(i,j)$ with $d_{\mathcal G}(i,j)\le K$.
	\end{proof}

	\begin{proof}[Proof of Proposition~\ref{prop:properties}] \emph{(i)} Recall that the spectrum of $\ma{L}$ is included in $[0,2]$ such that $p(\ma{L})\succ 0$. As $\ma{Q}\succ 0$, their sum is also positive-definite: the inverse exists. 
		
		\emph{(ii)} Set $\ma{M}=\ma{Q}^{-1/2}p(\ma{L})\ma{Q}^{-1/2}\succ 0$. One has $\ma{R} = (p(\ma{L})+\ma{Q})^{-1}\ma{Q} = \ma{Q}^{-1/2}(\ma{M}+\ma{I})^{-1}\ma{Q}^{1/2}$, 
		so $\ma{R}$ is similar to the symmetric matrix $(\ma{M}+\ma{I})^{-1}$ and therefore shares its eigenvalues. Since $\ma{M}\succ 0$, every eigenvalue of $\ma{M}+\ma{I}$ is strictly greater than $1$, hence every eigenvalue of $(\ma{M}+\ma{I})^{-1}$---and of $\ma{R}$---lies strictly in $(0,1)$.        
		
		\emph{(iii)} Note that $\ma{R}^\top = \ma{Q}(\ma{P}+\ma{Q})^{-1}$. $\ma{R} = \ma{R}^\top \iff (\ma{P}+\ma{Q})^{-1}\ma{Q} = \ma{Q}(\ma{P}+\ma{Q})^{-1} \iff \ma{P}\ma{Q} = \ma{Q}\ma{P} \iff \forall (i,j) \; (q_i - q_j)P_{ij} = 0$ as $\ma Q$ is diagonal. ($\Longrightarrow$) Let $\ma{R}$ be symmetric. For any adjacent nodes $i,j$, $d_{\mathcal G}(i,j) = 1 \le K$ (as $K>0$). Complete $K$-hop support implies $P_{ij} \neq 0$, forcing $q_i = q_j$. Propagating this along paths implies $q_i = q_j$ whenever $i$ and $j$ belong to the same connected component. ($\Longleftarrow$) Conversely, suppose $q$ is constant on each connected component. If $i,j$ are in the same component, $q_i - q_j = 0$. If they are in different components, $d_{\mathcal G}(i,j) = \infty > K$, so $P_{ij} = 0$ by localization. In both cases, $(q_i - q_j)P_{ij} = 0$, so $\ma{R}$ is symmetric.
	\end{proof}
	
	\begin{lemma}\label{lem}
		Consider the matrix $\ma{M}=p(\ma{L})\ma{R}$ and let $\vec{x}$ be a vector. Let $p_{\text{max}}=\max_k p(\lambda_k)$. The following upper bound is valid: $|[\ma{M}\vec{x}]_i|\leq p_\text{max}\|x\|_2$.
	\end{lemma}
	
	\begin{proof}[Proof of Lemma.]
		One has $\forall i$, $|[\ma{M}\vec{x}]_i|\leq \|\ma{M}\|_2 \|x\|_2$. Now, $p(\ma{L})$ being invertible, $\ma{M}$ can be re-written as $\ma{M}=\left(\ma{Q}^{-1}+p(\ma{L})^{-1}\right)^{-1}$: it is symmetric and positive-definite. As such, one has $\|\ma{M}\|_2= \lambda_{\text{max}}(\ma{M})=1/\lambda_{\text{min}}(\ma{Q}^{-1}+p(\ma{L})^{-1})$. To upper-bound this, we need to lower-bound the denominator. Noting that $\lambda_{\text{min}}(\ma{Q}^{-1}+p(\ma{L})^{-1})\geq\lambda_{\text{min}}(p(\ma{L})^{-1})=1/\lambda_{\text{max}}(p(\ma{L}))=1/p_{\text{max}}$ finishes the proof.
	\end{proof}
	
	\begin{proof}[Proof of Proposition~\ref{prop:harmonicity}]
		\emph{Feature preservation.} Let $\ma{M}=p(\ma{L})\ma{R}$. 
		Left-multiply $\ma{R}$ by $\ma{Q}^{-1}(p(\ma{L})+\ma{Q})$ and obtain 
		$\ma{R}-\ma{I}=-\ma{Q}^{-1}\ma{M}$. Thus $\forall i$:
		$$|(\ma{R}\vec{x})_i-x_i| 
		= |[\ma{Q}^{-1}\ma{M}\vec{x}]_i| 
		= q_i^{-1} |[\ma{M}\vec{x}]_i|\leq q_i^{-1} p_{\text{max}} \|\vec{x}\|_2$$ by Lemma~\ref{lem}. 
		
		\emph{Graph harmonicity.} Left-multiplying $\ma{R}\vec{x}$ by $p(\ma{L})+\ma{Q}$, one obtains: 
		$p(\ma{L})\vec{z} = \ma{Q}(\vec{x}-\vec{z})$, such that $|[p(\ma{L})\vec{z}]_i| = q_i|x_i-z_i|$. Now, one has:
		\begin{align*}
			q_i|x_i-z_i| = \sqrt{q_i}\sqrt{q_i(x_i-z_i)^2}.
		\end{align*}
		Denote by $F$ the functional $F(\vec{u})=\sum_j q_j(x_j-u_j)^2 + \vec{u}^\top p(\ma{L})\vec{u}$. One has:
		\begin{align*}
			q_i(x_i-z_i)^2\leq \sum_j q_j(x_j-z_j)^2 \leq F(\vec{z})
		\end{align*}
		as $\vec{z}^\top p(\ma{L})\vec{z}\geq 0$ (as $p(\ma{L})\succ 0$). By definition, $\vec{z}$ is the minimizer of the functional $F(\vec{u})$ such that: $F(\vec{z})\leq F(\vec{x})=\vec{x}^\top p(\ma{L})\vec{x}\leq p_\text{max}\|\vec{x}\|_2^2$. Putting everything back together, we finish the proof:
		\begin{align*}
			|[p(\ma{L})\vec{z}]_i|=q_i|x_i-z_i|\leq \sqrt{q_i}\sqrt{p_\text{max}}\|\vec{x}\|_
			2.
		\end{align*}
	\end{proof}
	
	\begin{proof}[Proof of Proposition~\ref{prop:density}]
		\emph{Density of the homogeneous family.}
		Fix a continuous target $h:[0,2]\to(0,1)$ and $\varepsilon>0$.
		Note that 
		both constants
		\[
		c := \min_{\lambda}\frac{h(\lambda)}{1-h(\lambda)} > 0
		\quad\text{and}\quad
		\mu := \min_{\lambda}\frac{1-h(\lambda)}{h(\lambda)} > 0
		\]
		are strictly positive (note that $c\mu\leq 1$).
		For a density argument it suffices to handle $\varepsilon < \mu$,
		so assume this without loss of generality.
		Choose any $q$ satisfying
		\[
		0 < q < \frac{c}{1+c\varepsilon};
		\]
		this is possible since $c>0$.
		Set $f(\lambda):=q(1-h(\lambda))/h(\lambda)$, so that $h(\lambda)=q/(f(\lambda)+q)$.
		Then $\min_\lambda f = q\mu$ and $\max_\lambda f = q/c$, giving
		\[
		\min_\lambda f - q\varepsilon = q(\mu-\varepsilon) > 0
		\quad\text{and}\quad
		1-\max_\lambda f - q\varepsilon = 1 - \frac{q}{c} - q\varepsilon = 1 - q\!\left(\tfrac{1}{c}+\varepsilon\right) > 0,
		\]
		where the second inequality uses $q < c/(1+c\varepsilon)$.
		Hence $f(\lambda)\in(0,1)$ for all $\lambda$.
		By the Weierstrass approximation theorem there exists a polynomial $p$
		with $\|p-f\|_\infty < q\varepsilon$.
		The two inequalities above guarantee $0 < p(\lambda) < 1$ on $[0,2]$.
		The homogeneous filter $g_q(\lambda)=q/(p(\lambda)+q)$ then satisfies
		\[
		|g_q(\lambda)-h(\lambda)|
		= \frac{q|f(\lambda)-p(\lambda)|}{(p(\lambda)+q)(f(\lambda)+q)}
		\leq \frac{q\cdot q\varepsilon}{q^2} = \varepsilon
		\]
		uniformly on $[0,2]$, establishing density.
		
		\emph{Strict generalization via non-commutativity.} 
		Let $\ma{P}=p(\ma{L})$. Since
		\[
		\ma{R}^{-1}
		=
		\ma{Q}^{-1}(\ma{P}+\ma{Q})
		=
		\ma{Q}^{-1}\ma{P}+\ma{I},
		\]
		and since $\ma{P}$ commutes with $\ma{L}$ and is invertible, we have
		\[
		[\ma{R},\ma{L}]=0
		\quad\Longleftrightarrow\quad
		[\ma{R}^{-1},\ma{L}]=0
		\quad\Longleftrightarrow\quad
		[\ma{Q}^{-1},\ma{L}]P=0
		\quad\Longleftrightarrow\quad
		[\ma{Q},\ma{L}]=0.
		\]
		where $[\ma{R},\ma{L}]=\ma{R}\ma{L}-\ma{L}\ma{R}$ denotes the commutator.
		For diagonal $\ma{Q}$,
		\[
		([\ma{Q},\ma{L}])_{ij}=(q_i-q_j)L_{ij},
		\]
		so $\ma{Q}$ commutes with $\ma{L}$ iff $q_i=q_j$ on every edge, i.e., iff $\ma{Q}$ is constant on each connected component. Thus, on a connected graph, any nonconstant $\ma{Q}$ gives an operator $\ma{R}$ that does not commute with $\ma{L}$.
		
		Finally, suppose that $\ma{R}=g(\ma{L})$ for some scalar function $g$.
		Then $\ma{R}$ commutes with $\ma{L}$, so by the previous paragraph $\ma{Q}=q_c\ma{I}$ on each connected component $G_c$. 
		On the zero-eigenvector supported on $G_c$, $\ma{R}$ acts by the scalar
		\[
		\frac{q_c}{p(0)+q_c}.
		\]
		But $g(\ma{L})$ acts as the same scalar $g(0)$ on the whole nullspace of $\ma{L}$. Since $p(0)>0$, the map
		\[
		q\mapsto \frac{q}{p(0)+q}
		\]
		is strictly increasing, so all $q_c$ must be equal. 
		Hence $\ma{Q}=q\ma{I}$. 
		Therefore, $\ma{R}$ can be written as a scalar spectral filter $g(\ma{L})$ only in the homogeneous case.
	\end{proof}
	
	\begin{proof}[Proof of Proposition~\ref{prop:receptive}]
		\emph{(i)} Define $\ma{M}=p(\ma{L})+\ma{Q}$. Since $\ma{R}_{ij}=(\ma{M}^{-1})_{ij}\,q_j$ with $q_j>0$, the zero
		pattern of $\ma{R}$ is that of $\ma{M}^{-1}$.
		Both $p(\ma{L})$ and $\ma{Q}$ are block-diagonal with respect to
		connected components, hence so are $\ma{M}$ and
		$\ma{M}^{-1}$, giving $\ma{R}_{ij}=0$ between components.
		
		\emph{(ii) Increasing degree-one case.}
		Write the increasing degree-one polynomial as $p(\lambda)=a+b\lambda$ with $b> 0$ (and necessarily $a>0$ as $0<p(\lambda)<1$ must hold at $\lambda=0$), so that $\ma{M}=p(\ma{L})+\ma{Q}=a\ma{I}+b\ma{L}+\ma{Q}=b(\ma{L}+\ma{Q}')$ with $\ma{Q}'=\frac{1}{b}\left(\ma{Q}+a\ma{I}\right)\succ0$. 
		On each connected component, $\ma{L}+\ma{Q}'$ is a symmetric positive definite matrix with non-positive off-diagonal entries, and it is irreducible when the component has more than one node. 
		Hence, by the standard inverse-positivity of irreducible nonsingular $M$-matrices,
		\[
		\left((\ma{L}+\ma{Q}')^{-1}\right)_{ij}>0
		\]
		for all $i,j$ in the same connected component. 
		The singleton case is immediate. 
		Since:
		\[
		\ma{R} = \ma{M}^{-1}\ma{Q} = \frac{1}{b}(\ma{L}+\ma{Q}')^{-1}\ma{Q},
		\]
		we get, for every pair $i,j$ within the same connected component:
		\[
		\ma{R}_{ij} = \frac{q_j}{b} \left((\ma{L}+\ma{Q}')^{-1}\right)_{ij} >0.
		\]
		
		\emph{(iii) Other cases.} Let us first treat the simple case $i=j$. One has: $\ma{R}_{ii}=\left(\ma{M}^{-1}\right)_{ii}q_i$ with $\ma{M} = p(\ma L) + \ma{Q}$. As $\ma{M}$ is positive definite, so is $\ma{M}^{-1}$. A well-known fact of  positive definite matrices is that their diagonal entries are always positive. As $q_i$ is also positive, we have that $\forall i, \ma{R}_{ii}>0$. 
		
		Now, let $\mathcal{C}$ be the set of all ordered pairs $(i,j)$ of distinct vertices belonging to the same connected component of $\mathcal{G}$. Fix a pair $(i,j) \in \mathcal{C}$.  Let the entries $q_1, \dots, q_n$ of the diagonal matrix $\ma{Q}$ be treated as independent formal variables. By Cramer's rule, the Tikhonov entry is given by
		\[
		\ma{R}_{ij} = \bigl(\ma{M}^{-1}\bigr)_{ij}\, q_j = \frac{\ma{C}_{ji}(\ma{M})\, q_j}{\det(\ma{M})},
		\]
		where $\ma{C}_{ji}(\ma{M})$ is the $(j,i)$-th cofactor of $\ma{M}$. By the Leibniz formula for determinants, both $\det(\ma{M})$ and the cofactors $\ma{C}_{ji}(\ma{M})$ are multivariate polynomials in $\mathbb{R}[q_1, \dots, q_n]$. Since $\ma{P} \succ 0$ and $\ma{Q} \succ 0$, $\ma{M}$ is positive definite, ensuring $\det(\ma{M}) > 0$ for all positive diagonal $\ma{Q}$.
		
		The cofactor is given by $\ma{C}_{ji}(\ma{M}) = (-1)^{i+j} \det(\ma{M}_{\setminus j, \setminus i})$, where $\ma{M}_{\setminus j,\setminus i}$ denotes the submatrix of $\ma{M}$ obtained by deleting row $j$ and column $i$. This submatrix is indexed by rows $\mathcal{V}\setminus\{j\}$ and columns $\mathcal{V}\setminus\{i\}$, and each variable $q_k$ ($k \neq i,j$) appears exactly once, in the entry corresponding to the original diagonal position $(k,k)$ of $\ma{M}$, as part of the linear term $\ma{P}_{kk}+q_k$.
		
		Since $i$ and $j$ lie in the same connected component, there exists a shortest path in the \emph{$K$-hop graph} (\ie, the graph whose edges connect pairs of vertices at graph distance at most $K$) from $i$ to $j$. Let $i = v_0, v_1, \dots, v_m = j$ be such a shortest path, so that $d_{\mathcal{G}}(v_r, v_{r+1}) \le K$ for all $0 \le r < m$, and $m$ is minimal. By the complete $K$-hop support hypothesis, each off-diagonal entry along this path satisfies $\ma{P}_{v_r, v_{r+1}} \neq 0$.
		
		Moreover, by minimality of the path, for any $r, s$ with $|r - s| > 1$ we have $d_{\mathcal{G}}(v_r, v_s) > K$: otherwise, replacing the sub-path $v_r, \dots, v_s$ by the single hop $v_r \to v_s$ would yield a strictly shorter path, contradicting minimality. By $K$-localization, it follows that $\ma{P}_{v_r, v_s} = 0$ whenever $|r-s|>1$.\\	
		We now exhibit a monomial in the Leibniz expansion of $\det(\ma{M}_{\setminus j, \setminus i})$ that is \emph{not canceled} by any other term. Consider the permutation
		\[
		\sigma^*: \quad v_r \longmapsto v_{r+1} \quad (0 \le r \le m-1), \qquad k \longmapsto k \quad (k \in \mathcal{V} \setminus \{v_0, \dots, v_m\}),
		\]
		which is a valid bijection from $\mathcal{V}\setminus\{j\} = \mathcal{V}\setminus\{v_m\}$ to $\mathcal{V}\setminus\{i\} = \mathcal{V}\setminus\{v_0\}$. This permutation selects: the off-diagonal constants $\ma{P}_{v_r, v_{r+1}}$ for the path vertices; and the linear diagonal terms $(\ma{P}_{kk} + q_k)$ for all $k \in \mathcal{V} \setminus \{v_0, \dots, v_m\}$. Choosing the variable $q_k$ from each such diagonal factor yields the monomial
		\[
		\gamma_{ij} \cdot \prod_{k \,\notin\, \{v_0, \dots, v_m\}} q_k,
		\qquad \text{where} \quad
		\gamma_{ij} = \operatorname{sgn}(\sigma^*)(-1)^{i+j} \prod_{r=0}^{m-1} \ma{P}_{v_r, v_{r+1}} \neq 0.
		\]
		This monomial involves precisely the variable set $\mathcal{S} = \{q_k \mid k \notin \{v_0, \dots, v_m\}\}$ and has degree $|\mathcal{S}| = n - (m+1)$.
		
		We claim no other permutation $\tau \neq \sigma^*$ contributes a term containing $\prod_{k \in \mathcal{S}} q_k$. Any such $\tau$ must also fix all $k \in \mathcal{V} \setminus \{v_0,\dots,v_m\}$ (to select the diagonal factor $\ma{P}_{kk}+q_k$ and thus include $q_k$), meaning $\tau$ must act as a bijection from $\{v_0,\dots,v_{m-1}\}$ to $\{v_1,\dots,v_m\}$ on the path vertices. But since $P_{v_r,v_s}=0$ for $|r-s|>1$, the only bijection from $\{v_0,\dots,v_{m-1}\}$ to $\{v_1,\dots,v_m\}$ that yields a nonzero product $\prod_r \ma{M}_{v_r,\tau(v_r)}$ is the forward shift $\sigma^*$ itself. In other words, $\sigma^*$ is the unique permutation contributing the monomial $\gamma_{ij}\prod_{k \in \mathcal{S}}q_k$: this term cannot be canceled and $C_{ji}(\ma{M})$ is not identically zero as a polynomial in $\mathbb{R}[q_1, \dots, q_n]$.\\
		For each pair $(i,j) \in \mathcal{C}$, define the exceptional set
		\[
		\mathfrak{E}_{ij} = \bigl\{\ma{Q} \in \mathbb{R}^n_{>0} \;\big|\; \ma{C}_{ji}(\ma{M}) = 0\bigr\},
		\]
		which is the zero-set of a nonzero polynomial, hence an algebraic variety of Lebesgue measure zero in $\mathbb{R}^n$. Define the global exceptional set
		\[
		\mathfrak{E} = \bigcup_{(i,j) \in \mathcal{C}} \mathfrak{E}_{ij}.
		\]
		Since $\mathcal{G}$ is finite, $|\mathcal{C}| \le n(n-1)$, so $\mathfrak{E}$ is a finite union of measure-zero sets and thus has Lebesgue measure zero.
		
		For any positive diagonal matrix $\ma{Q} \notin \mathfrak{E}$, we have $\ma{C}_{ji}(\ma{M}) \neq 0$ for all $(i,j) \in \mathcal{C}$ simultaneously. Since $q_j > 0$ and $\det(\ma{M}) > 0$, it follows that
		\[
		\ma{R}_{ij} = \frac{\ma{C}_{ji}(\ma{M})\,q_j}{\det(\ma{M})} \neq 0
		\]
		for every pair $(i,j)$ in the same connected component, completing the proof.
	\end{proof}
	
	\begin{proof}[Proof of Proposition~\ref{prop:decay}]
		Set $\ma{S}=\ma{I}+\ma{Q}^{-1/2}p(\ma{L})\,\ma{Q}^{-1/2}$ such that 
		$\ma{Q}^{1/2}\ma{S}\,\ma{Q}^{1/2}=p(\ma{L})+\ma{Q}$, giving 
		$\ma{R}=\ma{Q}^{-1/2}\ma{S}^{-1}\ma{Q}^{1/2}$, hence 
		$|\ma{R}_{ij}|=\sqrt{q_j/q_i}\;|(\ma{S}^{-1})_{ij}|$.
		The matrix $\ma{S}$ is SPD and is $K$-localized: $\ma{S}_{ij}=0$ whenever 
		$d_{\mathcal{G}}(i,j)>K$.
		Write $d=d_{\mathcal{G}}(i,j)$.
		By the Demko--Moss--Smith bound~\cite{demko1984decay},
		\[
		|(\ma{S}^{-1})_{ij}| \leq 2\|\ma{S}^{-1}\|_2\, \rho_0^{d/K}
		\]
		where $\rho_0=(\sqrt{\kappa}-1)/(\sqrt{\kappa}+1)$ and $\kappa=\mathrm{cond}(\ma{S})$.
		Using $\mathrm{arctanh}(x)\geq x$ for $x\in(0,1)$, $\rho_0\leq e^{-2/\sqrt{\kappa}}$. 
		Finally, since $\ma{S}\succeq\ma{I}$, $\|\ma{S}^{-1}\|_2\leq 1$, yielding $|(\ma{S}^{-1})_{ij}| \leq 2e^{-\frac{2}{K\sqrt{\kappa}}d}$ and the result.
	\end{proof}

	\begin{proof}[Proof of Proposition~\ref{prop:structural_faithfulness}]
		Write $\ma{P}=p(\ma{L})$. 
		Since $p(\lambda)>0$ on $[0,2]$, we have $P\succ0$, hence $P$ is invertible. 
		Moreover $\ma{R}^{-1} = \ma{Q}^{-1}(P+\ma{Q}) = \ma{Q}^{-1}P+\ma{I}$.
		If two positive diagonal matrices $\ma{Q}_1,\ma{Q}_2$ produce the same operator $\ma{R}$, then $\ma{Q}_1^{-1}P+\ma{I} = \ma{Q}_2^{-1}P+\ma{I}$, so $(\ma{Q}_1^{-1}-\ma{Q}_2^{-1})P=0$.
		Since $P$ is invertible, $\ma{Q}_1^{-1}=\ma{Q}_2^{-1}$, hence $\ma{Q}_1=\ma{Q}_2$.
	\end{proof}
	
	\begin{proof}[Proof of Proposition~\ref{prop:identifiability}]
		Let $\ma{P}_j = p_j(\ma{L})$ and $\ma{R}$ be the common operator. Inverting $\ma{R} = (\ma{P}_j + \ma{Q}_j)^{-1}\ma{Q}_j$ we obtain $\ma{R}^{-1} = \ma{Q}_j^{-1}\ma{P}_j + \ma{I}$, 
		hence $
		\ma{Q}_1^{-1}\ma{P}_1 = \ma{Q}_2^{-1}\ma{P}_2$.
		Let
		\[
		\ma{\Delta} = \ma{Q}_2 \ma{Q}_1^{-1} = \mathrm{diag}(\delta_1,\dots,\delta_n), 
		\quad \delta_i > 0.
		\]
		Then $\ma{P}_2 = \ma{\Delta}\,\ma{P}_1.$ 
		
		Since $\ma{P}_1$ is invertible (as $p_1(\lambda)>0$ for $\lambda\in[0,2]$), $\ma{\Delta}=\ma{P}_2\ma{P}_1^{-1}$.
		Both $\ma{P}_2$ and $\ma{P}_1^{-1}$ commute with $\ma{L}$, hence $\ma{\Delta}$ commutes with $\ma{L}$. Since $\ma{\Delta}$ is diagonal, $(\delta_i-\delta_j)L_{ij}=0$ for all $i,j$, so $\delta_i$ is constant on each connected component.
		Write this constant as $\alpha_c$.
		
		Now let $\vec{u}_c\neq \vec{0}$ be an eigenvector of $\ma{L}$ associated to eigenvalue $0$, supported on component $\mathcal{G}_c$. Since $\ma{P}_j=p_j(\ma{L})$,
		\[
		\ma{P}_2\vec{u}_c=p_2(0)\vec{u}_c, \qquad \ma{P}_1\vec{u}_c=p_1(0)\vec{u}_c.
		\]
		Using $\ma{P}_2=\ma{\Delta}\ma{P}_1$ on this component gives $p_2(0)\vec{u}_c=\alpha_c p_1(0)\vec{u}_c$.
		Since $p_1(0)>0$, all $\alpha_c=p_2(0)/p_1(0)$ are equal. 
		Denoting the common value by $\alpha$, we obtain
		\[
		\ma{Q}_2=\alpha\ma{Q}_1, \qquad \ma{P}_2=\alpha\ma{P}_1.
		\]
	\end{proof}
	
	\begin{proof}[Proof of Proposition~\ref{prop:multichannel_id}]
		Equal pre-activation outputs up to a permutation of channel blocks means 
		$\ma{R}'_j\ma{H}=\ma{R}_{\pi(j)}\ma{H}$ for some 
		permutation $\pi$. Since 
		$\ma{H}$ has rank $n$, it has a right inverse.
		Hence
		$\ma{R}'_j=\ma{R}_{\pi(j)}$ for each $j$. 
		Applying Proposition~\ref{prop:identifiability} to each channel 
		gives
		$\ma{Q}'_j=\alpha_{j}\,\ma{Q}_{\pi(j)}$ with $p'_j(\ma{L})=\alpha_{j}\,p_{\pi(j)}(\ma{L})$.
	\end{proof}
	
	\section{Choice of hyperparameters}
	\label{app:hp}
	\noindent\textbf{Common settings.} We use the Adam optimizer. We use one channel with a degree $K=5$ for the polynomial $p$. Skip connections are implemented in the $\ma{Q}$-network. The $\ma{Q}$-network ends with a 2-layer MLP head. $q_\text{min}=1\mathrm{e}{-}10$, $q_\text{max}=1\mathrm{e}10$. No dropout. 
	
	\noindent\textbf{How the other hyperparameters were chosen. }
	For all experiments except ECHO-diam, hyperparameters were found by sweeping (not exhaustively however -- due to computation time constraints) (i) for the Q-network: ChebNet/ARMA, hidden dimension $\in [4,8,16]$, number of layers $\in [3,5,10,15,20]$, chebychev order $\in [3,10]$ for ChebNet, number of inner layers and number of stacks $\in [3,5]$ for ARMA, the initialization\footnote{for this, we initialize the last layer of the MLP head to $0$ and set its bias to the wanted value} of $q\in [0.01, 0.1, 1.0]$ (ii) for the initial shape of the polynomial: close to $\lambda/2$ or close to flat; (iii) for the rest: hidden dimension $\in [4, 8, 16]$, patience $\in [50,150,250]$, learning rate $\in [5\mathrm{e}{-}4, 1\mathrm{e}{-}3, 5\mathrm{e}{-}3]$. Batch sizes were chosen from the original papers (for the CSBM dataset, we tried $256$ and $512$). For the global pooling: mean, sum (both without normalization), mean+max+sum with layer normalization, and sum+sumsq (the sum of squares, only for CSBM) with batch normalization. For PCG: tolerance and maximum number of iterations were not swept (1e-6 and 30, respectively, are in general very good for many tasks).
	
	For ECHO-diam, we swept 
	(i) for the Q-network: ChebNet/ARMA, hidden dimension $\in [4,8,16]$, number of layers $\in [3,5,10]$, chebychev order $\in [3,10]$ for ChebNet, number of inner layers and number of stacks $\in [3,5]$ for ARMA, the initialization value for $q: \in [0.001,0.01]$; (ii) for the initial shape of the polynomial: close to $\lambda/2$ or close to flat; (iii) for the rest: hidden dimension $\in[64,128,256,512]$, patience $\in [50,150]$, learning rate $\in [$5e-3, 1e-2$]$. For the global pooling: mean (without normalization), mean+max+sum with layer normalization. PCG tolerance and maximum number of iterations were swept: PCG tolerance $\in [$1e-6, 1e-10, 1e-14$]$, maximum number of iterations $\in [20,30,40,50]$ (best results were obtained for the smallest tolerance and the largest number of iterations: the task really needs to increase the reach of PCG in order to better approximate long-distance interactions). 
	
	Overall, we estimate that 3 full weeks of GPU time is the total time of computation of the experiments reported here. 
	
	\noindent\textbf{Colors-3} (Fig.~\ref{fig:colors3}).
	Initialization for the coefficients of the polynomial: close to $\lambda/2$. 
	Hidden dimension: 16. $\ma{Q}$-network: a ChebNet of 3 layers with degree-3 polynomials in each, hidden dimension of 8. Initialization for $q_i$: 1.0. 
	Pooling: mean+sum+max (concatenated) with a layer normalization. Batch size = 64. Patience = 50. PCG: tolerance is set to 1e-6, and the maximum number of iterations is $30$. Learning rate = 5e-3.
	
	\noindent\textbf{Clique distance} (Fig.~\ref{fig:triangles}, top).
	Initialization for the coefficients of the polynomial: close to the constant function equal to $0.5$. 
	Input features: the all-one vector. Hidden dimension: 8. $\ma{Q}$-network: a ChebNet of 5 layers with degree-3 polynomials in each, hidden dimension of 8. Initialization for $q_i$: 0.1. 
	Pooling: mean+sum+max (concatenated) with a layer normalization. Batch size = 128. Patience = 150. PCG: tolerance is set to 1e-6, and the maximum number of iterations is $30$. Learning rate = 5e-3.
	
	\noindent\textbf{Triangles} (Fig.~\ref{fig:triangles}, bottom).
	Initialization for the coefficients of the polynomial: close to $\lambda/2$. 
	Input features: the all-one vector. Hidden dimension: 8. $\ma{Q}$-network: a deep ARMA network of 20 layers with 3 inner layers and 3 stacks each, hidden dimension of 8. Initialization for $q_i$: 0.01. 
	Pooling: mean+sum+max (concatenated) with a layer normalization. Batch size = 128. Patience = 250. PCG: tolerance is set to $1e-6$, and the maximum number of iterations is $30$. Learning rate = 1e-3.
	
	\noindent\textbf{Triangles: second explanation} (Fig.~\ref{fig:triangles2}). We only list what is different from above.
	Here, the polynomial is fixed to a function very close to $\lambda/2$ (see caption of figure). 
	Hidden dimension: 16. 
	Pooling: sum with no normalization. 
	
	\noindent\textbf{CSBM} (Fig.~\ref{fig:csbm}).
	Initialization for the coefficients of the polynomial: close to $\lambda/2$. 
	Hidden dimension: 16. $\ma{Q}$-network: a ChebNet of 3 layers with degree-10 polynomials in each, hidden dimension of 8. Initialization for $q_i$: 0.1. 
	Pooling: sum+sumsq (concatenated) with a batch normalization. Batch size = 128. Patience = 50. PCG: tolerance is set to 1e-6, and the maximum number of iterations is $30$. Learning rate = 5e-4.
	
	\noindent\textbf{ECHO-diam} (Fig.~\ref{fig:echo}).
	Hidden dimension: 256 (we do as in~\cite{miglior2026can}, where they add a linear embedding layer prior to the network to test).
	Initialization for the coefficients of the polynomial: close to $\lambda/2$. 
	$\ma{Q}$-network: a ChebNet of 2 layers with degree-3 polynomials in each, hidden dimension of 8. Initialization for $q_i$: 0.001. 
	Pooling: mean with no normalization. Batch size = 512. Patience = 150. PCG: tolerance is set to 1e-14, and the maximum number of iterations is set to $50$ (number of iterations is always the stopping criterion in practice). Learning rate = 1e-2.
	
	\noindent\textbf{ECHO-diam: second explanation} (Fig.~\ref{fig:echo2}). We only list what is different from above.
	$\ma{Q}$-network: a ChebNet of 4 layers. 
	Pooling: mean+sum+max (concatenated) with a layer normalization. 
	
	\section{Background on the CSBM}\label{sec:csbm_background}
	
	The Contextual Stochastic Block Model (CSBM) combines two well-studied random models: the stochastic block model (SBM) for graph structure, and a Gaussian covariate model for node features. We consider graphs on $n$ nodes (even) split into two blocks of equal size $n/2$, with ground-truth labels $v_i\in\{-1,+1\}$.
	
	\paragraph{Graph structure: SBM.} For each pair of nodes $(i,j)$, the edge is drawn independently as a Bernoulli with parameter $p_\mathrm{in}$ if $i,j$ belong to the same block and $p_\mathrm{out}$ otherwise. Setting the average degree $\bar{d}$ and the signal parameter $\lambda\in[0,\sqrt{\bar{d}}]$:
	\[
	p_\mathrm{in} = \frac{1}{n}\bigl(\bar{d}+\lambda\sqrt{\bar{d}}\bigr),
	\qquad
	p_\mathrm{out} = \frac{1}{n}\bigl(\bar{d}-\lambda\sqrt{\bar{d}}\bigr).
	\]
	The block structure becomes detectable when $\lambda$ is large. The fundamental limit for recovering the communities (up to a constant fraction of misclassified nodes, with high probability as $n\to\infty$) is $\lambda=1$: weak recovery is possible if and only if $\lambda>1$~\cite{decelle_asymptotic_2011,massoulie_community_2014}.
	
	\paragraph{Node features: covariate model.} With feature dimension $p=n/\gamma$, node $i$ has feature vector $\vec{f}_i\in\mathbb{R}^p$ drawn as
	\[
	\vec{f}_i = \sqrt{\tfrac{\mu}{n}}\,v_i\,\vec{u} + \vec{z}_i,
	\]
	where $\vec{u}\sim\mathcal{N}(\vec{0},\ma{I}_p/p)$ is a shared random direction, $\vec{z}_i\sim\mathcal{N}(\vec{0},\ma{I}_p/p)$ is noise, and $\mu>0$ is the feature signal strength. Weak recovery from features alone is possible if and only if $\mu/\sqrt{\gamma}>1$~\cite{deshpande_contextual_2018}.
	
	\paragraph{Combined CSBM threshold.} Jointly using both graph and features, the information-theoretic threshold for weak recovery is~\cite{deshpande_contextual_2018}:
	\[
	\lambda^2 + \frac{\mu^2}{\gamma} > 1.
	\]
	This curve in the $(\lambda,\,\mu/\sqrt{\gamma})$ plane separates the recoverable from the unrecoverable regime, and is the red quarter-circle shown in Figure~\ref{fig:csbm}.
	
	\paragraph{Classification task.} The Tikhonov layer is trained to distinguish two distributions of graphs:
	\begin{itemize}
		\item \textbf{Class~1 (signal):} graphs drawn from CSBM$(\lambda,\bar{d},\gamma,\mu)$, with community-aligned features.
		\item \textbf{Class~0 (null):} features drawn with $\mu=0$ (no community signal); graph drawn from SBM$(\lambda,\bar{d})$ then randomly edge-rewired to preserve the degree distribution (removing any detectable community structure while keeping the degree sequence).
	\end{itemize}
	A model succeeding at this task must (i) jointly exploit structure and features, especially in the zone where $\lambda^2 + \frac{\mu^2}{\gamma} > 1$ and $\lambda^2<1$ and $\frac{\mu^2}{\gamma} < 1$; (ii)~automatically detect when $\mu$ is small that the features are useless noise; (ii)~automatically detect when $\lambda$ is small that the topology is useless noise.

	\section{Additional experiments}
	\label{add_expe}
	We plot, in Figures~\ref{fig:echo2} and~\ref{fig:triangles2}, additional explanations obtained with other choices of hyperparameters (see Appendix~\ref{app:hp} to see the exact changes in hyperparameters), for, respectively, the ECHO-diam and the Triangles datasets. In Figure~\ref{fig:echo2}, we see an explanation that is harder to understand than the one shown in the main text (the filter is not only penalizing high frequencies but also low frequencies, in quite a symmetrical way), but boasts a better test MAE and matches the MAE reported in the original paper for the best Transformers architecture.
	
	In Figure~\ref{fig:triangles2}, we show another explanation that is different in spirit than the one shown in the main text. In this example, the filter was not learned and set to a function very close to $p(\lambda)=\lambda/2$ (see caption), thus mimicking the traditional Tikhonov denoising setting. In the results of the run, we see that the triangle nodes are still different from the others, but this time associated with low values of $q$ compared to large values of the other nodes. The network is not counting the number of nodes in the triangle; it is diffusing the features that are in the triangle. Both explanations are equally valid (and obtain the same almost-perfect test accuracy).

	\begin{figure}[t]
		\centering
		\includegraphics[width=\textwidth]{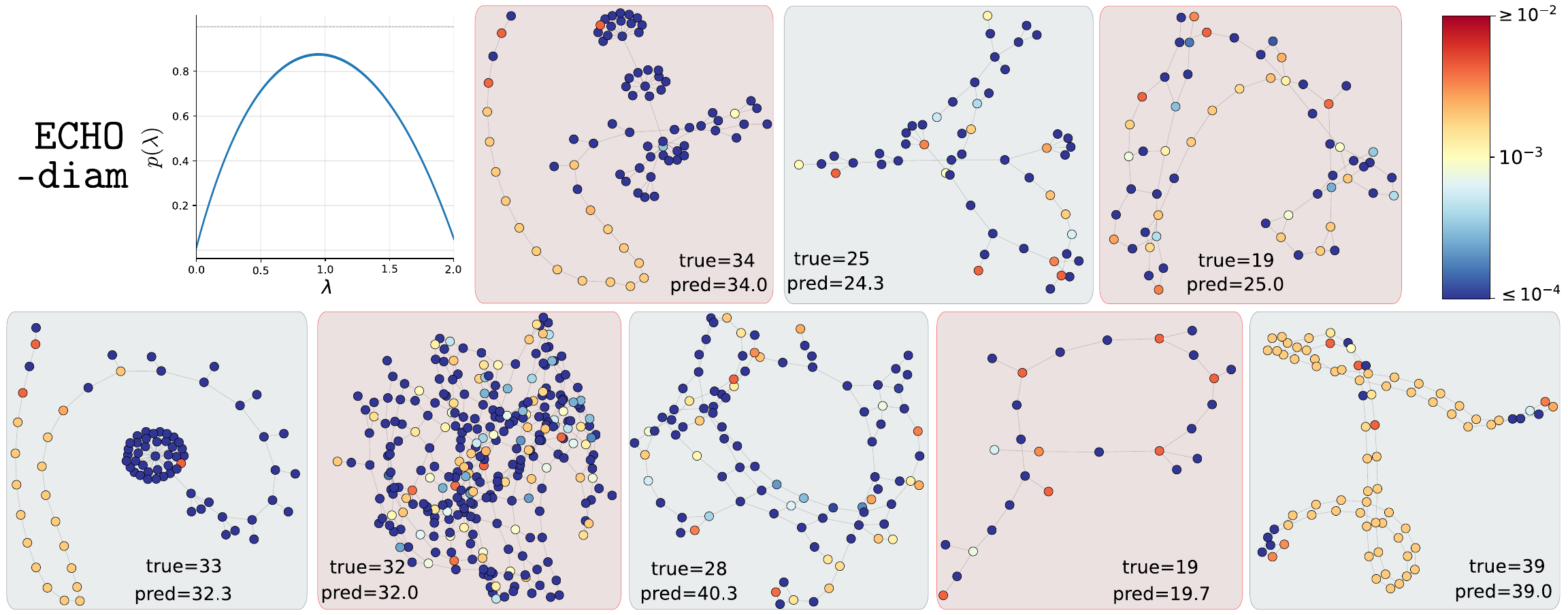}
		\caption{\textbf{ECHO-diam benchmark.} Another explanation, obtained with other hyperparameters. The learned polynomial (top-left) along with the learned $q$-values plotted on 8 example graphs (log scale; blue: $q_i\leq 10^{-4}$, red: $q_i\geq 10^{-2}$). The true diameter is written, along with the prediction of the Tikhonov layer. 
			Best run (in terms of validation loss) out of 5 runs (with different seeds). The average test MAE is $1.12$ (the run shown has a test MAE of $1.05$). Recall that the MAE here is the average error of prediction of the diameter. }
		\label{fig:echo2}
	\end{figure}
	
	\begin{figure}[t]
		\centering
		\includegraphics[width=\textwidth]{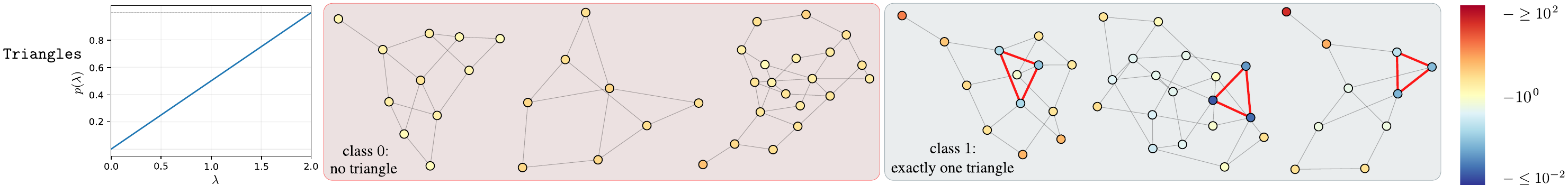}
		\caption{\textbf{Triangles benchmark.} Another explanation, obtained with other hyperparameters. The polynomial is here fixed (not learned) to $p(\lambda)=\epsilon+(\frac{1}{2}-\epsilon)\lambda$ with $\epsilon=1$e-$3$ (left). One obtains the learned $q$-values plotted on 6 example graphs (log scale; blue: $q_i\leq 10^{-2}$, red: $q_i\geq 10^{2}$): 3 from each class. 
			Best run (in terms of validation loss) out of 5 runs (with different seeds). The average test accuracy is $95\%$ (the run shown has test accuracy of $99.5\%$).}
		\label{fig:triangles2}
	\end{figure}

	
\end{document}